\documentclass[format=acmsmall, review=false]{acmart}
\usepackage{acm-ec-20}
\usepackage{booktabs} 
\usepackage[ruled]{algorithm2e} 
\usepackage{algorithmic}

\SetAlFnt{\small}
\SetAlCapFnt{\small}
\SetAlCapNameFnt{\small}
\SetAlCapHSkip{0pt}
\IncMargin{-\parindent}

\usepackage{tikz}

\setcitestyle{authoryear}

\allowdisplaybreaks

\newcommand{\R}{\mathbb{R}}

\newcommand{\NW}{\mathtt{WNW}}
\newcommand{\MNW}{\mathtt{MWNW}}
\newcommand{\efi}{\mathtt{EF1}}
\newcommand{\wef}{\mathtt{WEF}}
\newcommand{\wefi}{\mathtt{WEF1}}
\newcommand{\wwefi}{\mathtt{WWEF1}}
\newcommand{\propi}{\mathtt{PROP1}}
\newcommand{\WP}{\mathtt{WPROP}}
\newcommand{\WPi}{\mathtt{WPROP1}}
\newcommand{\WPc}{\mathtt{WPROP}c}
\newcommand{\WM}{\mathtt{WMMS}}
\newcommand{\po}{\mathtt{PO}}
\newcommand{\wefc}{\mathtt{WEF}c}
\newcommand{\wwefc}{\mathtt{WWEF}c}

\newcommand{\MBB}{\mathrm{MBB}}
\newcommand{\wmin}{\mathrm{wmin}}
\newcommand{\wmax}{\mathrm{wmax_{-1}}}
\newcommand{\poly}{\mathrm{poly}}

\newcommand{\bfp}{\boldsymbol{p}}
\newcommand{\wpefi}{\mathtt{WpEF1}}
\newcommand{\argmin}{\mbox{argmin}}

\newcommand{\ceil}[1]{\lceil #1 \rceil }
\newcommand{\floor}[1]{\lfloor #1 \rfloor }
\newcommand{\calI}{\mathcal{I}}

\title[Weighted Envy-Freeness in Indivisible Item Allocation]{Weighted Envy-Freeness in Indivisible Item Allocation}

\author{Mithun Chakraborty}
\affiliation{%
  \institution{University of Michigan, USA}}
  
\author{Ayumi Igarashi}
\affiliation{%
  \institution{National Institute of Informatics, Japan}}
  
\author{Warut Suksompong}
\affiliation{%
  \institution{National University of Singapore, Singapore}}
  
\author{Yair Zick}
\affiliation{%
  \institution{University of Massachusetts, Amherst, USA}}

\begin{abstract}
We introduce and analyze new envy-based fairness concepts for agents with \emph{weights} that quantify their entitlements in the allocation of indivisible items. 
We propose two variants of weighted envy-freeness up to one item (WEF1): \emph{strong}, where envy can be eliminated by removing an item from the envied agent's bundle, and \emph{weak}, where envy can be eliminated either by removing an item (as in the strong version) or by replicating an item from the envied agent's bundle in the envying agent's bundle. 
We show that for additive valuations, an allocation that is both Pareto optimal and strongly WEF1 always exists and can be computed in pseudo-polynomial time; moreover, an allocation that maximizes the weighted Nash social welfare may not be strongly WEF1, but always satisfies the weak version of the property. 
Moreover, we establish that a generalization of the round-robin picking  sequence algorithm produces in polynomial time a strongly WEF1 allocation for an arbitrary number of agents; for two agents, we can efficiently achieve both strong WEF1 and Pareto optimality by adapting the adjusted winner procedure.
Our work highlights several aspects in which weighted fair division is richer and more challenging than its unweighted counterpart.
\end{abstract}

\begin{document}

\maketitle

\section{Introduction}

The fair allocation of resources to interested parties is a central issue in economics and has increasingly received attention in computer science in the past few decades~\cite{brams1996fair,Moulin03,Thomson16,markakis2017approximation,Moulin19}.
The problem has a wide range of applications, from reaching divorce settlements~\cite{brams1996fair} and dividing land~\cite{SegalhaleviNiHa17} to sharing apartment rent~\cite{GalMaPr17}.
\emph{Envy-freeness} is one of the most commonly studied fairness criterion in the literature; it stipulates that all agents find their assigned bundle to be the best among all bundles in the allocation \citep{Foley67,Varian74}.

Envy-freeness is a compelling notion when all agents have equal entitlements---indeed, in a standard envy-free allocation, no agent would rather take the place of another agent with respect to the assigned bundles.
However, in many division problems, agents may have varying claims on the resource.
For instance, consider a facility that has been jointly funded by three investors---Alice, Bob, and Charlie---where Alice contributed $3/5$ of the construction expenses while Bob and Charlie contributed $1/5$ each.
One could then expect Alice to envy either Bob or Charlie if she does not value her share at least three times as much as each of the latter two investors' share when they divide the usage of the facility.
Besides this interpretation as the \emph{cost} of participating in the resource allocation exercise, the weights may also represent other publicly known and accepted measures of entitlement such as \emph{eligibility} or \emph{merit}.
A prevalent example is inheritance division, wherein closer relatives are typically more entitled to the bequest than distant ones.
Likewise, different countries have differing entitlements when it comes to apportioning humanitarian aid.
Envy-freeness can be naturally extended to the general setting in which agents have \emph{weights} designating their entitlements.
When the resource to be allocated is infinitely divisible (e.g., time to use a facility, or land in a real estate), it is known that a weighted envy-free allocation exists for any set of agents' weights and valuations \cite{RobertsonWe98,zeng2000approximate}.

In this article, we initiate the study of weighted envy-freeness for the ubiquitous setting where the resource consists of \emph{indivisible} items.
Indeed, inheritance division usually involves discrete items such as real estate, cars, and jewelry; similarly, facility usage is often allocated in fixed time slots (e.g., hourly).
Since envy-freeness cannot always be fulfilled even in the canonical setting without weights, for example when all agents agree that one particular item is more valuable than the remaining items combined, recent works have focused on identifying relaxations of envy-freeness that can be satisfied in the case of equal entitlements.
The most salient of these approximations is perhaps \emph{envy-freeness up to one item (EF1)}: for any two agents $i$ and $j$, if agent $i$ envies agent $j$, then we can eliminate this envy by removing a single item from $j$'s bundle~\citep{budish2011combinatorial}.
\citet{lipton2004approximately} showed that an EF1 allocation exists and can be computed efficiently for any number of agents with monotone valuations.\footnote{EF1 is also remarkable for its robustness: it can be satisfied under cardinality constraints~\citep{BiswasB18} and connectivity constraints~\citep{bilo2019almost,Suksompong19}, has a relatively low price in terms of social welfare \citep{BeiLuMa19,barman2020optimal}, and is computable using few queries~\cite{OhPrSu19}.}
Our goal in this work is to extend EF1 to the general case with arbitrary entitlements, and explore the relationship of these extensions to other important justice criteria such as proportionality and Pareto optimality.
The richness of the weighted setting will be evident throughout our work; in particular, we demonstrate that while some protocols from the unweighted setting can be generalized to yield strong guarantees, others are less robust and cease to offer desirable properties upon the introduction of weights.

\subsection{Our Contributions}
We assume that agents have positive (not necessarily rational) weights representing their entitlements and, with the exception of Propositions~\ref{prop:wef-1-identical}, \ref{mwnw_notef1}, and \ref{non_transfer}, that they are endowed with additive valuation functions.
We begin in Section~\ref{sec:prelims} by proposing two generalizations of EF1 to the weighted setting: (strong) weighted envy-freeness up to one item ($\wefi$) and weak weighted envy-freeness up to one item ($\wwefi$). 
While $\wefi$ may appear as the more natural extension, we argue that it can impose a highly demanding constraint when the weights vastly differ, so that $\wwefi$ is a useful alternative.
In Section~\ref{sec:wef1}, we focus on two classical EF1 protocols. 
On the one hand, we show that the envy cycle elimination algorithm of \citet{lipton2004approximately} does not extend to the weighted setting except in the special case of identical valuations.
On the other hand, we construct a weight-based picking sequence which allows us to compute a $\wefi$ allocation efficiently---this generalizes a folklore result from the unweighted setting.
The analysis of this algorithm is significantly more involved than for the unweighted version and requires making intricate algebraic manipulations.
Nevertheless, the algorithm itself is simple both to explain and to implement, so we believe that it is suitable for maintaining fairness in practice. 

In Sections \ref{sec:wef1_po} and \ref{sec:wwef1_po}, we
examine the interplay between fairness and Pareto optimality.
For two agents, we exhibit that a weighted variant of the adjusted winner procedure allows us to compute an allocation that is both $\wefi$ and Pareto optimal in polynomial time---our algorithm provides a natural discretization of the classical procedure, which was designed for the divisible item setting.
We then show by adapting an algorithm of \citet{barman2018finding} that a Pareto optimal and $\wefi$ allocation is guaranteed to exist and can be found in pseudo-polynomial time for any number of agents.
Furthermore, we prove that while an allocation with maximum weighted Nash welfare may fail to satisfy $\wefi$, such an allocation is both Pareto optimal and $\wwefi$, thereby generalizing an important result of \citet{caragiannis2016unreasonable}.
Our proof for the $\wwefi$ result follows a similar outline as that of Caragiannis et al., but we need to make a case distinction based on the comparison between weights.
We continue our investigation in Section~\ref{sec:wef1_propMMS} by exploring the relationship between weighted envy-freeness and the weighted versions of other fairness concepts---in particular, we find that several relationships from the unweighted setting break down---and illustrate through experiments in Section~\ref{sec:expts} that envy-freeness is often harder to satisfy in the presence of weights than otherwise.
Finally, we conclude in Section~\ref{sec:nonadd} by discussing some obstacles that we faced when trying to extend our ideas and results beyond additive valuation functions; specifically, we show that a $\wwefi$ allocation may not exist when agents have non-additive (submodular) valuations.

\subsection{Related Work}
\label{sec:relatedwork}
There is a long line of work on the fair division of indivisible items; see, e.g., the surveys by  \citet{bouveret2016fair} and \citet{markakis2017approximation} for an overview.
Prior work on the fair allocation of indivisible items to asymmetric agents has tackled fairness concepts that are not based on envy. \citet{farhadi2019fair} introduced \emph{weighted maximin share (WMMS)} fairness, a generalization of an earlier fairness notion of \citet{budish2011combinatorial}. 
\citet{aziz2019weighted} explored WMMS fairness in the allocation of indivisible \emph{chores}---items that, in contrast to goods, are valued negatively by the agents---where agents' weights can be interpreted as their shares of the workload. 
\citet{babaioff2019fair} studied competitive equilibrium for agents with different budgets.
Recently, \citet{aziz2019polynomial} proposed a polynomial-time algorithm for computing an allocation of a pool of goods and chores that satisfies both Pareto optimality and \emph{weighted proportionality up to one item (WPROP1)} for agents with asymmetric weights. 
Unequal entitlements have also been considered in the context of divisible items with respect to \emph{proportionality} \citep{Barbanel95,brams1996fair,cseh2018complexity,RobertsonWe98,segalhalevi19cake}.
We remark here that (weighted) proportionality is a strictly weaker notion than (weighted) envy-freeness under additive valuations.
However, while PROP1 is also implied by EF1 in the unweighted setting, the relationship between the corresponding weighted notions is much less straightforward, as we demonstrate in Section~\ref{sec:wef1_propMMS}.

In addition to expressing the entitlement of individual agents, weights can also be applied to settings where each agent represents a group of individuals \citep{benabbou2019fairness,benabbou2020diversity}---here, the \emph{size} of a group can be used as its weight.\footnote{Note that in this model, each group has a valuation function that represents the overall preference of its members. Other group fairness notions do not assume the existence of such aggregate functions and instead take directly into account the preferences of the individual agents in each group \citep{conitzer2019group,kyropoulou2019almost,segal2019democratic,segal2021how}.}
Specifically, in the model of \citet{benabbou2020diversity}, groups correspond to ethnic groups (namely, the major ethnic groups in Singapore, i.e., Chinese, Malay, and Indian).
Maintaining provable fairness guarantees amongst the ethnic groups is highly desirable; in fact, it is one of the principal tenets of Singaporean society.

\section{Preliminaries}
\label{sec:prelims}
Throughout the article, given a positive integer $r$, we denote by $[r]$ the set $\{1,2,\dots,r\}$. 
We are given a set $N = [n]$ of {\em agents}, and a set $O = \{o_1,\dots,o_m\}$ of {\em items} or {\em goods}. Subsets of $O$ are referred to as {\em bundles}, and each agent $i \in N$ has a {\em valuation function} $v_i:2^O \to \R_{\geq 0}$ over bundles; the valuation function for every $i \in N$ is {\em normalized} (i.e., $v_i(\emptyset) = 0$) and {\em monotone} (i.e., $v_i(S) \leq v_i(T)$ whenever $S \subseteq T$). We denote $v_i(\{o\})$ simply by $v_i(o)$ for any $i \in N$ and $o \in O$.

An {\em allocation} $A$ of the items to the agents is a collection of $n$ disjoint bundles $A_1,\dots,A_n$ such that $\bigcup_{i \in N} A_i \subseteq O$; the bundle $A_i$ is allocated to agent $i$ and $v_i(A_i)$ is agent $i$'s {\em realized valuation} under $A$. Given an allocation $A$, we denote by $A_0$ the set $O\setminus \left(\bigcup_{i \in N} A_i\right)$, and its elements are referred to as {\em withheld items}. An allocation $A$ is said to be \emph{complete} if $A_0 = \emptyset$, and \emph{incomplete} otherwise.

In our setting with different entitlements, each agent $i \in N$ has a fixed weight $w_i\in\mathbb{R}_{>0}$; these weights regulate how agents value their own allocated bundles relative to those of other agents, and hence bear on the overall (subjective) fairness of an allocation. More precisely, we define the \emph{weighted envy} of agent $i$ towards agent $j$ under an allocation $A$ as $\max\left\{0,\frac{v_i(A_j)}{w_j}-\frac{v_i(A_i)}{w_i}\right\}$. An allocation is \emph{weighted envy-free} ($\wef$) if no agent has positive weighted envy towards another agent.
Intuitively, agent~$i$ being weighted envy-free towards agent~$j$ means that $i$'s valuation for her share $A_i$, given that $i$'s entitlement is $w_i$, is at least as high as $i$'s valuation for $A_j$ if $i$'s entitlement were $w_j$.
Weighted envy-freeness reduces to traditional envy-freeness when $w_i=w$, $\forall i \in N$ for some positive real constant $w$. Since a complete envy-free allocation does not always exist, it follows trivially that a complete $\wef$ allocation may not exist in general.
We briefly remark here that with indivisible items, it is possible to define variations of weighted envy-freeness---for example, if $w_i=1$ and $w_j=2$, one could require that agent~$j$'s bundle can be divided into two parts neither of which agent~$i$ finds more valuable than her own bundle.
Nevertheless, the definition that we use is mathematically natural and can be directly applied to arbitrary (not necessarily rational) weights.

We now state the main definitions of our article, which naturally extend \emph{envy-freeness up to one item (EF1)} \citep{lipton2004approximately,budish2011combinatorial} to the weighted setting.

\begin{definition}
An allocation $A$ is said to be \emph{(strongly) weighted envy-free up to one item} ($\wefi$) if for any pair of agents $i,j$ with $A_j\neq\emptyset$, there exists an item $o \in A_j$ such that \[\frac{v_i(A_i)}{w_i} \ge  \frac{v_i(A_j \setminus \{o\})}{w_j}.\]
More generally, $A$ is said to be \emph{weighted envy-free up to $c$ items} ($\wefc$) for an integer $c \ge 1$ if for any pair of agents $i,j$, there exists a subset $S_c \subseteq A_j$ of size at most $c$ such that \[\frac{v_i(A_i)}{w_i} \ge \frac{v_i(A_j\setminus S_c)}{w_j}.\]
\end{definition}

\begin{definition}
An allocation $A$ is said to be \emph{weakly weighted envy-free up to one item} ($\wwefi$) if for any pair of agents $i,j$ with $A_j\neq\emptyset$, there exists an item $o \in A_j$ such that 
\[\text{either } \frac{v_i(A_i)}{w_i} \ge  \frac{v_i(A_j \setminus \{o\})}{w_j} \text{ or } \frac{v_i(A_i\cup\{o\})}{w_i} \ge  \frac{v_i(A_j)}{w_j}.\]
More generally, $A$ is said to be \emph{weakly weighted envy-free up to $c$ items} ($\wwefc$) for an integer $c \ge 1$ if for any pair of agents $i,j$, there exists a subset $S_c \subseteq A_j$ of size at most $c$ such that \[\text{either }\frac{v_i(A_i)}{w_i} \ge \frac{v_i(A_j\setminus S_c)}{w_j} \text{ or } \frac{v_i(A_i\cup S_c)}{w_i} \ge  \frac{v_i(A_j)}{w_j}.\]
\end{definition}

In other words, an allocation is $\wefi$ if any (weighted) envy from an agent~$i$ towards another agent~$j$ can be eliminated by removing a single item from $j$'s bundle.
Similarly, $\wwefi$ requires that any such envy can be eliminated by either removing an item from $j$'s bundle or adding a copy of an item from $j$'s bundle to $i$'s bundle.

A valuation function $v:2^O \to \R_{\geq 0}$ is said to be \emph{additive} if $v(S)= \sum_{o \in S} v(o)$ for every $S \subseteq O$. We will assume additive valuations for most of the article; this is a very common assumption in the fair division literature and offers a tradeoff between expressiveness and simplicity~\citep{BouveretLe16,caragiannis2016unreasonable,KurokawaPrWa18}. 
Under this assumption, both $\wefi$ and $\wwefi$ reduce to EF1 in the unweighted setting. 
Moreover, one can check that under additive valuations, an allocation satisfies $\wwefi$ if and only if for any pair of agents $i,j$ with $A_j\neq\emptyset$, there exists an item $o \in A_j$ such that
\[\frac{v_i(A_i)}{w_i} \ge \frac{v_i(A_j)}{w_j} - \frac{v_i(o)}{\min\{w_i,w_j\}}.\]

The criterion $\wefi$ can be criticized as being too demanding in certain circumstances, when the weight of the envied agent is much larger than that of the envying agent. To illustrate this, consider a problem instance where agent $1$ has an additive valuation function and is indifferent among all items taken individually, e.g., $v_1(o)=1$ for every $o \in O$. Now, if $w_1=1$ and $w_2=100$, then eliminating one item from agent $2$'s bundle reduces agent $1$'s weighted valuation of this bundle by merely $0.01$. As such, we may trigger a substantial adverse effect on the overall welfare of the allocation by aiming to eliminate agent $1$'s weighted envy towards agent $2$. This line of thinking was our motivation for introducing the weak weighted envy-freeness concept. 
We also note that $\wwefi$ can be viewed as a stronger version of what one could refer to as ``transfer weighted envy-freeness up to one item'': agent $i$ is \emph{transfer weighted envy-free up to one item} towards agent~$j$ under the allocation $A$ if there is an item $o \in A_j$ that would eliminate the weighted envy of $i$ towards $j$ upon being transferred from $A_j$ to $A_i$, i.e., $v_i(A_i \cup \{o\}) \ge \frac{w_i}{w_j} \cdot v_i(A_j \setminus \{o\})$.

In addition to fairness, we often want our allocation to satisfy an efficiency criterion. One important such criterion is Pareto optimality. An allocation $A'$ is said to \emph{Pareto dominate} an allocation $A$ if $v_i(A'_i) \ge v_i(A_i)$ for all agents $i \in N$ and $v_j(A'_j) > v_j(A_j)$ for some agent $j \in N$. 
An allocation is \emph{Pareto optimal} (or $\po$ for short) if it is not Pareto dominated by any other allocation. 

Allocations maximizing the {\em Nash welfare}---defined as $\mathtt{NW}(A) := \prod_{i \in N} v_i(A_i)$---are known to offer strong guarantees with respect to both fairness and efficiency in the unweighted setting \citep{caragiannis2016unreasonable}.
For our weighted setting, we define a natural extension called \emph{weighted Nash welfare}---$\NW(A) := \prod_{i \in N} v_i(A_i)^{w_i}$. 
Since it is possible that the maximum attainable $\NW(A)$ is $0$, we define a \emph{maximum weighted Nash welfare} or $\MNW$ allocation along the lines of \citet{caragiannis2016unreasonable} as follows: given a problem instance, we find a maximum subset of agents, say $N_{\max} \subseteq N$, to which we can allocate bundles of positive value, and compute an allocation to the agents in $N_{\max}$ that maximizes\footnote{There can be multiple maximum subsets $N_{\max}$ having the same cardinality but different maximum weighted Nash welfare.
Our main positive result for $\MNW$ (Theorem~\ref{thm:mwnw}) holds for all such subsets $N_{\max}$.} $\prod_{i \in N_{\max}} v_i(A_i)^{w_i}$.
To see why the notion of $\MNW$ makes intuitive sense, consider a setting where agents have a value of $1$ for each item; furthermore, assume that the number of items is exactly $\sum_{i =1}^n w_i$. In this case, one can verify (using standard calculus) that an allocation maximizing $\MNW$ assigns to agent $i$ exactly $w_i$ items. Indeed, following the interpretation of $w_i$ as the number of members of group $i$ (see Section~\ref{sec:relatedwork}), the expression $v_i(A_i)^{w_i}$ can be thought of as each member of group $i$ deriving the same value from the set $A_i$; the group's overall Nash welfare is thus $v_i(A_i)^{w_i}$.

We also examine the extent to which weighted envy-freeness relates to the weighted versions of two other key fairness notions: \emph{proportionality} and the \emph{maximin share guarantee}.

An allocation $A$ is said to be \emph{weighted proportional} ($\WP$) if for every agent $i \in N$, it holds that $v_i(A_i) \ge \frac{w_i}{\sum_{j \in N} w_j}v_i(O)$. For a positive integer $c$, it is \emph{weighted proportional up to $c$ items} ($\WP c$) if for every $i \in N$, there exists a subset of items not allocated to $i$, i.e., $S_c \subseteq O \setminus A_i$, of size at most $c$ such that $v_i(A_i) \ge \frac{w_i}{\sum_{j \in N} w_j}\cdot v_i(O) - v_i(S_c)$; this is a natural extension of the (weighted) PROP1 concept \cite{conitzer2017fair,aziz2019polynomial}.\footnote{For any $a<b$, $\WP a$ implies $\WP b$ but the converse does not hold. Indeed, the former follows directly from the definition. For the latter, consider a problem instance with $n=2$ and $O=\{o_1,\dots,o_b\}$, weights $w_1>b-1$ and $w_2=1$, and identical, additive valuation functions such that $v_i(o)=1$ for all $i\in N$ and $o\in O$. It can be verified that the allocation that gives all items to agent~$2$ is $\WP b$ but not $\WP a$. }

Let $\Pi(O)$ denote the collection of all (ordered) $n$-partitions of the set of items $O$, or, in other words, the collection of all complete allocations of $O$ to $n$ agents. Then, the \emph{weighted maximin share} \cite{farhadi2019fair} of agent $i$  is defined as: $$\WM_i := \max_{(A_1,A_2,\dots,A_n) \in \Pi(O)} \min_{j \in N} \frac{w_i}{w_j} v_i(A_j).$$ 
An allocation $A$ is called $\WM$ if $v_i(A_i) \ge \WM_i$ for every $i \in N$.
More generally, for any approximation ratio $\alpha \in (0,1]$, $A$ is called $\alpha$-$\WM$ if $v_i(A_i) \ge \alpha \cdot \WM_i$ for every $i \in N$.

\section{$\wefi$ allocations}
\label{sec:wef1}

We commence our exploration of weighted envy-freeness by considering extensions of two standard methods for producing $\efi$ allocations in the unweighted setting: the \emph{envy cycle elimination algorithm} and the \emph{round-robin algorithm}.
As we will see, these two procedures experience contrasting fortunes in the presence of weights: while the idea of eliminating envy cycle fundamentally fails, the round-robin algorithm admits an elegant generalization that can take into account arbitrary entitlements of the agents.
 
\subsection{Envy Cycle Elimination Algorithm}

Before we discuss the envy cycle elimination algorithm of \citet{lipton2004approximately}, let us briefly recap how it works in the unweighted setting.
The algorithm allocates one item at a time in an arbitrary order.
It also maintains an ``envy graph'', which captures the envy relation between the agents with respect to the (incomplete) allocation at each stage.
The next item is allocated to an unenvied agent, and any envy cycle that forms as a result is eliminated by letting each agent take the bundle of the agent that she envies. 
This cycle elimination step allows the algorithm to ensure that there is an unenvied agent to whom it can allocate the next item.
 
As far as envy in the traditional sense is concerned, what an agent actually ``envies'' is an allocated bundle regardless of who owns that bundle. However, both the subjective valuations of allocated bundles and the relative weights interact in non-trivial ways to determine weighted envy. It is easy to see that weighted envy of $i$ towards $j$ does not imply traditional envy of $i$ towards $j$, and vice versa. A crucial implication is that even if agent $i$'s bundle is replaced with the bundle of an agent $j$ towards whom $i$ has weighted envy, $i$'s realized valuation, and hence the ratio of her realized valuation to her weight, may decrease as a result. 
Indeed, consider a problem instance with $n=2$ and $O=\{o_1,o_2,o_3\}$, weights $w_1=3$ and $w_2=1$, and identical, additive valuation functions such that  $v_i(o)=1$ for all $i \in N$ and $o \in O$. Under the complete allocation with $A_1=\{o_1,o_2\}$, agent $1$ has weighted envy towards agent $2$ since $v_1(A_2)/w_2=1/1=1>2/3=v_1(A_1)/w_1$, but agent $1$ would not prefer to replace $A_1$ with $A_2$ since that reduces her realized valuation from $2$ to $1$. On the other hand, agent $2$ could benefit from replacing $A_2$ with $A_1$ even though she does not have weighted envy towards agent $1$. As such, the natural extension of the envy cycle elimination algorithm, where an edge exists from agent $i$ to agent $j$ if and only if $i$ has weighted envy towards $j$, does not guarantee a complete $\wefi$ or even $\wwefi$ allocation.

\begin{proposition}\label{prop:wef-1-diff}
	The weighted version of the envy cycle elimination algorithm may not produce a complete $\wwefi$ allocation, even in a problem instance with two agents and additive valuations.	
\end{proposition}

\begin{proof}
Consider a problem instance with $n=2$ and $m=12$, weights $w_1=1$ and $w_2=2$, and valuation functions defined by
	\begin{align*}
	v_1(o_r) = \begin{cases}
	1 &\text{for $r=1$};\\
	0.1 &\text{for $r=12$};\\
	0.21 &\text{otherwise};
	\end{cases}\quad \text{and} \quad 
	v_2(o_r) = \begin{cases}
	1.1 &\text{for $r=1$};\\
	0.1 &\text{for $r=12$};\\
	0.2 &\text{otherwise}.
	\end{cases}
	\end{align*}
	
	Suppose that the weighted envy cycle elimination algorithm iterates over $o_1,o_2,\dots,o_{12}$, and starts by allocating $o_1$ to agent $1$ due to, say, lexicographic tie-breaking. At this point, agent $2$ has weighted envy towards agent $1$ and not vice versa; moreover, this condition persists until items $o_2, \dots, o_{10}$ have all been allocated to agent $2$. At this point, item $o_{11}$ also goes to agent $2$, resulting in valuations $v_1(A_1)=v_1(o_1)=1$ and $v_2(A_2)=v_2(\{o_2,\dots,o_{11}\})=2$.
	Agent $2$ still has weighted envy towards agent $1$ since
	$v_2(A_2)/w_2=1<1.1/1=v_2(A_1)/w_1$; 
	on the other hand, agent~$1$ also develops weighted envy towards agent $2$ since 
	$v_1(A_2)/w_2=1.05>1=v_1(A_1)/w_1$.
	Thus, there is a cycle in the induced weighted envy graph. For an unweighted envy graph, we would ``de-cycle'' the graph at this point by swapping bundles over the cycle and that would still maintain the invariant that the allocation is $\efi$. However, if we swap the bundles in this example so that the new allocated bundles are $A'_1=A_2=\{o_2,\dots,o_{11}\}$ and $A'_2=A_1=\{o_1\}$, agent $2$ will end up having (weak) weighted envy up to more than one item towards agent $1$ since 
	$v_2(A'_2)/w_2=1.1/2=0.55$ and $v_2(A'_1 \setminus \{o\})/w_1=(0.2 \times 9)/1=1.8$ for every $o \in A'_1$, and this weighted envy persists no matter how we allocate $o_{12}$.\footnote{Another interesting feature of this example is that the two agents have \emph{commensurable} valuations, i.e., both agents have the same valuation for the entire collection of items $O$: $\sum_{o \in O} v_1(o)=\sum_{o \in O} v_2(o) = 3.2$. This shows that the negative result of Proposition~\ref{prop:wef-1-diff} holds even if we impose the additional restriction of commensurability on the valuation functions.}
\end{proof}

By replacing each of the items $o_2,\dots,o_{11}$ with $c$ smaller items of equal value, one can check that the envy cycle elimination algorithm cannot even guarantee $\wwefc$ for any fixed $c$.
In spite of this negative result, the algorithm does work in the special case where the agents all have the same valuations.

\begin{proposition}\label{prop:wef-1-identical}
	The weighted version of the envy cycle elimination algorithm produces a complete $\wefi$ allocation whenever agents have identical (not necessarily additive) valuations, i.e., $v_i(S)=v(S)$ for some $v:2^O \to \R_{\ge 0}$, $\forall i \in N$, $\forall S \subseteq O$.
\end{proposition}

\begin{proof}
	By construction of the algorithm,
	the (incomplete) allocation at the end of each iteration is guaranteed to be $\wefi$ as long as we can find an agent, say $i$, towards whom no other agent has weighted envy at the beginning of the iteration: we give the item under consideration to agent $i$ and thus any resulting weighted envy towards $i$ can be eliminated by removing this item. If there is no unenvied agent, then the weighted envy graph consists of at least one cycle; however, under identical valuations, the envy graph cannot have cycles. 
	Indeed, suppose that agents $1,2,\ldots,\ell$ form a cycle (in that order) for some $\ell \in [n]$. 
	Since agents have identical valuations, it must be that $v(A_1)/w_1 < v(A_2)/w_2 < \dots < v(A_\ell)/w_\ell < v(A_1)/w_1$, a contradiction. 
\end{proof}

\subsection{Picking Sequence Protocols}\label{sec:picking}

We now turn our attention to protocols that let agents pick their favorite item according to some predetermined sequence.
When all agents have equal weight and additive valuations, it is well-known that a round-robin algorithm, wherein the agents take turns picking an item in the order $1,2,\dots,n,1,2,\dots,n,\dots$, produces an $\efi$ allocation.
This is in fact easy to see: If agent $i$ is ahead of agent $j$ in the ordering, then in every ``round'', $i$ picks an item that she likes at least as much as $j$'s pick; by additivity, $i$ does not envy $j$.
On the other hand, if agent $i$ picks after agent $j$, then by considering the first round to begin at $i$'s first pick, it follows from the same argument that $i$ does not envy $j$ up to the first item that $j$ picks.

We show next that in the general setting with weights, we can construct a weight-dependent picking sequence which guarantees $\wefi$ for any number of agents and arbitrary weights. 
The resulting algorithm is efficient, intuitive and can be easily explained to a layperson, so we believe that it has a strong practical appeal.
Unlike in the unweighted case, however, the proof that the algorithm produces a fair allocation is much less straightforward and requires making several intricate arguments.

\begin{theorem}\label{thm:pickseq}
	For any number of agents with additive valuations and arbitrary positive real weights, there exists a picking sequence protocol that computes a complete $\wefi$ allocation in polynomial time.
\end{theorem}
\begin{algorithm}
\caption{Weighted Picking Sequence Protocol}         
\label{alg_pickseq}
\begin{algorithmic}[1]
\STATE Remaining items $\widehat{O} \leftarrow O$.
\STATE Bundles $A_i \leftarrow \emptyset$, $\forall i \in N$.
\STATE $t_i \leftarrow 0$, $\forall i \in N$. \textbf{/*number of times each agent has picked so far*/}\\
\WHILE{$\widehat{O} \neq \emptyset$}
\STATE $i^* \leftarrow \arg\min_{i \in N} \frac{t_i}{w_i}$, breaking ties lexicographically.
\STATE $o^* \leftarrow \arg\max_{o \in \widehat{O}} v_{i^*}(o)$, breaking ties arbitrarily.
\STATE $A_{i^*} \leftarrow A_{i^*} \cup \{o^*\}$.
\STATE $\widehat{O} \leftarrow \widehat{O} \setminus \{o^*\}$.
\STATE $t_{i^*} \leftarrow t_{i^*}+1$.
\ENDWHILE
\end{algorithmic}
\end{algorithm}

To prove Theorem~\ref{thm:pickseq}, we construct a picking sequence such that in each turn, an agent with the lowest weight-adjusted picking frequency picks the next item (Algorithm~\ref{alg_pickseq}). We claim that after the allocation of each item, for any agent $i$, every other agent is weighted envy-free towards $i$ up to the item that $i$ picked first. 

To this end, first note that choosing an agent who has had the minimum (weight-adjusted) number of picks thus far ensures that the first $n$ picks are a round-robin over all of the agents; in this phase, the allocation is obviously $\wefi$ since each agent has at most one item at any point. We will show that, after this phase, the algorithm generates a picking sequence over the agents with the following property: 
	\begin{lemma}\label{lem:pickseq}
		Consider an agent $i$ chosen by Algorithm~\ref{alg_pickseq} to pick an item at some iteration $t$, and suppose that this is not her first pick. Let $t_i$ and $t_j$ be the numbers of times agent $i$ and some other agent $j$ appear in the prefix of iteration $t$ in the sequence respectively, not including iteration $t$ itself. Then $\frac{t_j}{t_i} \ge \frac{w_j}{w_i}$.
	\end{lemma}	
	
\begin{proof}
	Since agent $i$ is picked at iteration $t$, it must be that $i \in \arg\min_{k \in N} \frac{t_k}{w_k}$. This means that $\frac{t_i}{w_i} \le \frac{t_j}{w_j}$, i.e.,  $\frac{t_j}{t_i} \ge \frac{w_j}{w_i}$ since $t_i > 0$.
\end{proof}

The property guaranteed by Lemma~\ref{lem:pickseq} is sufficient to ensure that the latest picker does not attract weighted envy up to more than one item towards herself after her latest pick:

\begin{lemma}\label{lem:each_pick_wef1}
	 Suppose that, for every iteration $t$ in which agent $i$ picks an item after her first pick, the numbers of times that agent $i$ and some other agent $j$ appear in the prefix of the iteration in the sequence, not including iteration $t$ itself---call them $t_i$ and $t_j$ respectively---satisfy the relation $\frac{t_j}{t_i} \ge \frac{w_j}{w_i}$. Then, in the partial allocation up to and including $i$'s latest pick, agent $j$ is weighted envy-free towards $i$ up to the first item $i$ picked.
\end{lemma}

\begin{figure}[!ht]
\centering
\begin{tikzpicture}[scale=0.85]
\draw [fill = gray!50] (3,2.4) rectangle (4,3.9);
\draw (4.5,2.4) rectangle (5.5,3.4);
\draw [fill = gray!50] (6,2.4) rectangle (7,3.9);
\draw (7.5,2.4) rectangle (8.5,3.4);
\draw (9,2.4) rectangle (10,3.4);
\node at (2,2) {$i$};
\node at (3.5,2) {$j$};
\node at (5,2) {$i$};
\node at (6.5,2) {$j$};
\node at (8,2) {$i$};
\node at (9.5,2) {$i$};
\node at (3.5,3.15) {$1$};
\node at (5,2.9) {$\frac{2}{3}$};
\node at (6.5,3.15) {$1$};
\node at (8,2.9) {$\frac{2}{3}$};
\node at (9.5,2.9) {$\frac{2}{3}$};
\end{tikzpicture}
\caption{Illustration of the intuition behind the proof of Lemma~\ref{lem:each_pick_wef1}. Here, $i<j$, $w_i = 3$, and $w_j = 2$. The rectangles represent the agents' buckets, and the numbers therein correspond to their capacities. Note that agent~$i$ does not receive a bucket in her first pick. Agent~$j$'s buckets are filled, while those of agent~$i$ are empty.}
\label{fig:water}
\end{figure}
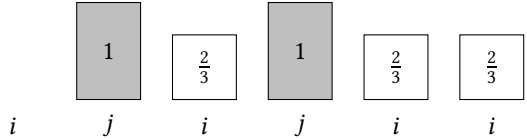

Before we prove Lemma~\ref{lem:each_pick_wef1}, we first provide a high-level intuition.
Recall the argument for the unweighted case at the beginning of Section~\ref{sec:picking}.
One way to visualize this argument is that when we consider envy from agent~$j$ towards agent~$i$, every time agent~$j$ picks an item, we give her a bucket with $1$ unit of water, while every time agent~$i$ picks an item from the second time onwards, we give her an empty bucket of capacity $1$.
Agent~$j$ is allowed to pour water from any of her buckets into any of $i$'s buckets that comes later in the sequence.
Since $j$ values an item that she picks at least as much as any item that $i$ picks in a later turn, in order to establish $\efi$, it suffices to show that $j$ can fill up all of $i$'s buckets using such operations.
A similar idea can be used in the weighted setting, except that in order to account for the weights, every time agent~$i$ picks after the first time, we give her an empty bucket of capacity $w_j/w_i$ units (see Figure~\ref{fig:water} for an example when $w_i > w_j$).
Note in particular that this bucket setup is entirely independent of the agents' valuations for the items.
However, unlike in the unweighted setting, where agent~$j$ can accomplish the task by simply pouring all the water from each of her buckets into $i$'s following bucket, in the weighted case, $j$ may need to pour water from a bucket into several of $i$'s buckets, even those coming after $j$'s subsequent bucket. 
The proof below formalizes this intuition.

\begin{proof}[Proof of Lemma~\ref{lem:each_pick_wef1}]
	Let $\gamma := \frac{w_j}{w_i}$. Consider any iteration $t$ in which agent $i$ is chosen after her first pick. Let agent $j$'s values for the items allocated to agent $i$ in the latter's second, third, $\dots$, $(t_i+1)^\mathrm{st}$ picks (the last one occurring at the iteration $t$ under consideration) be $\beta_1$, $\beta_2$, $\dots$, $\beta_{t_i}$ respectively. If $o^*$ is the first item picked by agent $i$ and $A^t$ the partial allocation up to and including iteration $t$, then clearly $v_j(A_i^t \setminus \{o^*\}) = \sum_{x=1}^{t_i} \beta_x$. Let the number of times agent $j$ appears in the prefix of agent $i$'s second pick be $\tau_1$; that between agent $i$'s second and third picks be $\tau_2$; $\dots$; that between agent $i$'s $t_i^\mathrm{th}$ and $(t_i+1)^\mathrm{st}$ picks be $\tau_{t_i}$. Let agent $j$'s values for the items she herself picked during phase $x\in[t_i]$ be $\alpha^x_1, \alpha^x_2, \dots, \alpha^x_{\tau_x}$ respectively, where the phases are defined as in the previous sentence. Then, $v_j(A^t_j)= \sum_{x=1}^{t_i} \sum_{y=1}^{\tau_x} \alpha^x_y$.
	Now, for $r\in[t_i]$, since $\sum_{x=1}^{r} \tau_x$ and $r$ are the numbers of times agents $j$ and $i$ appear in the prefix of the latter's $(r+1)^\mathrm{st}$ pick respectively, the condition of the lemma dictates that
	\begin{align}
	\sum_{x=1}^{r} \tau_x \ge r \gamma \quad \forall r \in [t_i].\label{spill_and_fresh}
	\end{align}
	Note that $\tau_1 \ge \gamma >0$; however, $\tau_x$ can be zero for $x \in \{2,3,\dots,t_i\}$---this corresponds to the scenario where agent $i$ picked more than once without agent $j$ picking in between.
	Moreover, every time agent $j$ was chosen, she picked one of the items she values the most among those available, including the items picked by agent $i$ later. Hence, if $\tau_x > 0$ for some $x\in [t_i]$, then
	\begin{align}
	\alpha^x_y &\ge \max\{\beta_x, \beta_{x+1}, \dots, \beta_{t_i}\} \quad \forall y \in [\tau_x]  \notag\\ 
	\Rightarrow \quad \sum_{y=1}^{\tau_x} \alpha^x_y &\ge \tau_x \max\{\beta_x, \beta_{x+1}, \dots, \beta_{t_i}\}.  \label{alphaxy}
	\end{align}
	Note that Inequality~\eqref{alphaxy} holds trivially if $\tau_x=0$ since both sides are zero; hence it holds for every $x \in [t_i]$.
	
	We claim that for each $r\in[t_i]$,
	\begin{align*}
	\sum_{x=1}^r \sum_{y=1}^{\tau_x} \alpha^x_y \ge \gamma\sum_{x=1}^r\beta_x + \left(\sum_{x=1}^r\tau_x - r\gamma\right)\max\{\beta_r, \beta_{r+1},\dots,\beta_{t_i}\}.
	\end{align*}
	To prove the claim, we proceed by induction on $r$. For the base case $r=1$, we have from Inequality~\eqref{alphaxy} that
	\begin{align*}
	\sum_{y=1}^{\tau_1} \alpha^1_y 
	&\ge \tau_1\max\{\beta_1,\beta_2,\dots,\beta_{t_i}\} \\
	&\ge \gamma\beta_1 + (\tau_1-\gamma)\max\{\beta_1,\beta_2,\dots,\beta_{t_i}\}.
	\end{align*}
	For the inductive step, assume that the claim holds for $r-1$; we will prove it for $r$.
	We have
	\begin{align*}
	\qquad\sum_{x=1}^r\sum_{y=1}^{\tau_x} \alpha^x_y &= \sum_{x=1}^{r-1}\sum_{y=1}^{\tau_x} \alpha^x_y + \sum_{y=1}^{\tau_r} \alpha^r_y \\
	&\ge \gamma\sum_{x=1}^{r-1}\beta_x + \left(\sum_{x=1}^{r-1}\tau_x - (r-1)\gamma\right)\max\{\beta_{r-1}, \beta_r,\dots,\beta_{t_i}\}  + \sum_{y=1}^{\tau_r} \alpha^r_y \\
	&\ge \gamma\sum_{x=1}^{r-1}\beta_x + \left(\sum_{x=1}^{r-1}\tau_x - (r-1)\gamma\right)\max\{\beta_{r-1}, \beta_r,\dots,\beta_{t_i}\}  + \tau_r\max\{\beta_r,\beta_{r+1},\dots,\beta_{t_i}\} \\
	&\ge \gamma\sum_{x=1}^{r-1}\beta_x + \left(\sum_{x=1}^{r-1}\tau_x - (r-1)\gamma\right)\max\{\beta_r, \beta_{r+1},\dots,\beta_{t_i}\}  + \tau_r\max\{\beta_r,\beta_{r+1},\dots,\beta_{t_i}\} \\
	&= \gamma\sum_{x=1}^{r-1}\beta_x + \left(\sum_{x=1}^{r}\tau_x - (r-1)\gamma\right)\max\{\beta_r, \beta_{r+1},\dots,\beta_{t_i}\}\\	
	&= \gamma\sum_{x=1}^{r-1}\beta_x + \gamma\max\{\beta_r,\beta_{r+1},\dots,\beta_{t_i}\} + \left(\sum_{x=1}^r\tau_x - r\gamma\right)\max\{\beta_r, \beta_{r+1},\dots,\beta_{t_i}\} \\
	&\ge \gamma\sum_{x=1}^{r-1}\beta_x + \gamma \beta_r + \left(\sum_{x=1}^r\tau_x - r\gamma\right)\max\{\beta_r, \beta_{r+1},\dots,\beta_{t_i}\} \\
	&= \gamma\sum_{x=1}^r\beta_x + \left(\sum_{x=1}^r\tau_x - r\gamma\right)\max\{\beta_r, \beta_{r+1},\dots,\beta_{t_i}\},
	\end{align*}	
	where the first inequality follows directly from the inductive hypothesis, the second from Inequality~\eqref{alphaxy}, and the third from the following facts: 
	\begin{align*}
		&\sum_{x=1}^{r-1}\tau_x - (r-1)\gamma \ge 0 && \text{due to Inequality~\eqref{spill_and_fresh}},\\
		\text{and} \quad  &\max\{\beta_{r-1}, \beta_r,\dots,\beta_{t_i}\} \ge \max\{\beta_r,\beta_{r+1},\dots,\beta_{t_i}\}.&&
	\end{align*}
	  This completes the induction and establishes the claim.
	
	Now, taking $r=t_i$ in the claim, we get
	\begin{align*}
	\sum_{x=1}^{t_i} \sum_{y=1}^{\tau_x} \alpha^x_y 
	&\ge \gamma\sum_{x=1}^{t_i}\beta_x + \left(\sum_{x=1}^{t_i}\tau_x-t_i\gamma\right)\beta_{t_i} \\
	&\ge \gamma\sum_{x=1}^{t_i}\beta_x,
	\end{align*}
	where we use Inequality~\eqref{spill_and_fresh} again for the second inequality.
	This implies that $v_j(A^t_j) \ge \frac{w_j}{w_i} \cdot v_j(A_i^t \setminus \{o^*\})$, i.e., agent $j$ is weighted envy-free towards agent $i$ up to one item (specifically, the first item picked by agent $i$), concluding the proof of the lemma and therefore the proof of correctness.
\end{proof}

With Lemmas~\ref{lem:pickseq} and \ref{lem:each_pick_wef1} in hand, we are now ready to prove Theorem~\ref{thm:pickseq}.

\begin{proof}[Proof of Theorem~\ref{thm:pickseq}]
It is easy to see that directly after an agent picks an item, her envy towards other agents cannot get any worse than before. Since the partial allocation after the initial round-robin phase is $\wefi$ and every agent is weighted envy-free up to one item towards every subsequent picker due to Lemmas~\ref{lem:pickseq} and~\ref{lem:each_pick_wef1}, the allocation is $\wefi$ at every iteration, and in particular at the end of the algorithm. This establishes the correctness of the algorithm.

For the time complexity, note that there are $O(m)$ iterations of the \textbf{while} loop. In each iteration, determining the next picker takes $O(n)$ time, while letting the picker pick her favorite item takes $O(m)$ time. Since we may assume that $m > n$ (otherwise it suffices to allocate at most one item to every agent), the algorithm runs in time $O(m^2)$.
\end{proof}

If $w_i$ equals a positive constant $w$ for every $i \in N$, then Algorithm~\ref{alg_pickseq} degenerates into the traditional round-robin procedure which is guaranteed to return an $\efi$ allocation for additive valuations, but may not be Pareto optimal; as such, Algorithm~\ref{alg_pickseq} may not produce a $\po$ allocation either. This is easily seen in the following example: $n=m=2$, $w_1=w_2=1$, $v_1(o_1)=v_1(o_2)=0.5$, $v_2(o_1)=0.8$, and $v_2(o_2)=0.2$. With lexicographic tie-breaking for both agents and items, our algorithm will give us $A_1=\{o_1\}$ and $A_2=\{o_2\}$, which is Pareto dominated by $A'_1=\{o_2\}$ and $A'_2=\{o_1\}$. 
On the other hand, if each agent has the same value for all items, the algorithm is equivalent to an apportionment method called \emph{Adams' method} \citep{BalinskiYo01}.\footnote{We are grateful to Ulrike Schmidt-Kraepelin for pointing out this connection.}
In the apportionment setting, agents correspond to states of a country, and items to seats in a parliament.
Since all seats are considered identical, the states can simply ``pick'' any seat from the remaining seats in apportionment, whereas for item allocation it is important that each agent picks her favorite item in her turn.

\section{$\wefi$ and $\po$ allocations}\label{sec:wef1_po}

As the picking sequence that we construct in Section~\ref{sec:picking} yields an allocation that is $\wefi$ but may fail Pareto optimality, our next question is whether $\wefi$  can be achieved in conjunction with the economic efficiency notion. 
We show that this is indeed possible, by generalizing the classic adjusted winner procedure for two agents and an algorithm of \citet{barman2018finding} for higher numbers of agents.

\subsection{Two Agents}
\label{sec:wef1_po_twoagents}

When agents have equal entitlements, it is known that fairness and efficiency are compatible: \citet{caragiannis2016unreasonable} showed that an allocation maximizing the Nash social welfare satisfies both $\po$ and $\efi$. 
Unfortunately, this approach is not applicable to our setting---we show that the $\MNW$ allocation may fail to be $\wefi$. 
In fact, we prove a much stronger negative result: for any fixed~$c$, the allocation may fail to be $\wefc$ even for two agents with identical valuations.

\begin{proposition}
Let $c$ be an arbitrary positive integer. There exists a problem instance with two agents having identical additive valuations for which any $\MNW$ allocation is not $\wefc$. 	
\end{proposition}

\begin{proof}
Suppose that $n=2$, and the weights are $w_1=1$ and $w_2=k$ for some positive integer~$k$ such that $\left(1+\frac{1}{k+c}\right)^k > 2$; such an integer $k$ exists because $\lim_{k\rightarrow\infty}\left(1+\frac{1}{k+c}\right)^k = e$.
Let $m=k+c+2$, so $O=\{o_1,o_2,\ldots,o_{k+c+2}\}$.
The agents have identical, additive valuations defined by $v_i(o)=1$ for all $i \in N$ and $o \in O$. 
Since $\left(1+\frac{1}{k+c}\right)^k > 2$, we have $1\cdot(k+c+1)^k > 2\cdot(k+c)^k$.
Moreover, for $2\leq i\leq k+c$, we have $$\left(1+\frac{1}{k+c+1-i}\right)^k > \left(1+\frac{1}{k+c}\right)^k > 2 > \frac{i+1}{i},$$ and so $i(k+c+2-i)^k > (i+1)(k+c+1-i)^k$.
This means that any $\MNW$ allocation $A$ must give one item to agent $1$, say $A_1=\{o_1\}$, and the remaining items to agent $2$, i.e., $A_2=\{o_2,\ldots,o_{k+c+2}\}$. However, even if we remove a set $S_c$ of at most $c$ items from $A_2$, we would still have $v_1(A_2\setminus S_c)/w_2\geq 1+1/k>1=v_1(A_1)/w_1$, so the allocation is not $\wefc$. 
\end{proof}

Given that a $\MNW$ allocation may not be $\wefi$ in our setting, a natural question is whether there is an alternative approach for guaranteeing the existence of a $\po$ and $\wefi$ allocation. 
We first show that this is indeed the case for two agents: we establish that such an allocation exists and can be computed in polynomial time for two agents, by adapting the classic adjusted winner procedure \cite{brams1996fair} to the weighted setting.
\begin{theorem}\label{thm:2agents}
	For two agents with additive valuations and arbitrary positive real weights, a complete $\wefi$ and $\po$ allocation always exists and can be computed in polynomial time.
\end{theorem}
\begin{algorithm}
\caption{Weighted Adjusted Winner Procedure}         
\label{alg_waw}                          
\begin{algorithmic}[1] 
\REQUIRE $\frac{v_1(o_1)}{v_2(o_1)} \ge \frac{v_1(o_2)}{v_2(o_2)} \ge \dots \ge \frac{v_1(o_m)}{v_2(o_m)}$ w.l.o.g.\\
\STATE $d \leftarrow 1$.
\WHILE{$\frac{1}{w_1}\sum_{r=1}^d v_1(o_r) < \frac{1}{w_2}\sum_{r=d+2}^m v_1(o_r)$}
\STATE $d \leftarrow d+1$.
\ENDWHILE
\STATE $A_1 \leftarrow \{o_1, \dots,o_d\}$.
\STATE $A_2 \leftarrow \{o_{d+1}, \dots,o_m\}$.
\end{algorithmic}
\end{algorithm}
\begin{proof}
	Assume first that both agents have positive values for all items, i.e., $v_1(o)>0$ and $v_2(o)>0$ for every $o \in O$; we will deal with the case where this does not hold later.
	We claim that the Weighted Adjusted Winner procedure as delineated in Algorithm~\ref{alg_waw} produces an allocation satisfying the theorem statement.
	
	First note that the left-hand side of the \textbf{while} loop condition is strictly increasing in $d$ and trivially exceeds the right-hand side for $d=m-1$;\footnote{For $d \ge m-1$, we set the right-hand side of the \textbf{while} loop condition to zero.} hence, there always exists $d \in [m-1]$ which satisfies the stop criterion and the loop terminates at the smallest such $d$. 
	
	\paragraph{$\wefi$ property:} If the \textbf{while} loop ends with $d=d^*$, let us denote $o_{d^*}$ by $o^*$. Then, by construction, 
	 $\frac{v_1(A_1)}{w_1} \ge \frac{v_1(A_2)-v_1(o_{d^*+1})}{w_2}$, which implies that agent $1$ is weighted envy-free towards agent $2$ up to one item (specifically, item $o_{d^*+1}$). 
	 
	 On the other hand, by construction, we also get that
	\begin{align}
	& \frac{v_1(A_1)-v_1(o^*)}{w_1} < \frac{v_1(A_2)}{w_2}.
	\label{ineq:waw}
	\end{align}
	Moreover, due to the ordering of the ratios, we have
	\begin{align*}
	&\frac{\sum_{r=1}^{d^*} v_1(o_r)}{\sum_{r=1}^{d^*} v_2(o_r)} \ge \frac{v_1(o^*)}{v_2(o^*)} \ge \frac{\sum_{r=d^*+1}^{m} v_1(o_r)}{\sum_{r=d^*+1}^{m} v_2(o_r)}\\
	\Rightarrow \quad & \frac{v_1(A_1)}{v_2(A_1)} \ge \frac{v_1(o^*)}{v_2(o^*)} \ge \frac{v_1(A_2)}{v_2(A_2)}\\
	\Rightarrow \quad & v_1(A_1) \ge \frac{v_1(o^*)}{v_2(o^*)} \cdot v_2(A_1);\\
	& v_1(A_2) \le \frac{v_1(o^*)}{v_2(o^*)} \cdot v_2(A_2).
	\end{align*}
	Combining with Inequality~\eqref{ineq:waw}, dividing through by $\frac{v_1(o^*)}{v_2(o^*)}$, and simplifying, we get
	\begin{align*}
	&\frac{v_2(A_1)-v_2(o^*)}{w_1} < \frac{v_2(A_2)}{w_2}.
	\end{align*}
	Thus, agent $2$ is also weighted envy-free towards agent $1$ up to one item (specifically, item $o^*$). 
	
	\paragraph{Pareto optimality: } First note that  no incomplete allocation can be Pareto optimal since the realized valuation of either agent could be strictly improved by augmenting her bundle with a withheld item (under our assumption that each agent values each item positively). 	
	Since the allocation $A$ produced by Algorithm~\ref{alg_waw} is complete, it suffices to show that it cannot be Pareto dominated by an alternative complete allocation $A'$. Any such complete allocation can be thought of as being generated by transferring items between $A_1$ and $A_2$. 
	
	Suppose that $v_1(A'_1) > v_1(A_1)$ for some complete allocation $A'$ different from $A$. Since $A_1 \cup A_2 = A'_1 \cup A'_2 = O$, this inequality implies that
	\begin{align}
		v_1(A'_1 \cap A_1) + v_1(A'_1 \cap A_2) &> v_1(A'_1 \cap A_1) + v_1(A'_2 \cap A_1)\notag\\
		\Rightarrow \quad  v_1(A'_1 \cap A_2) &>  v_1(A'_2 \cap A_1).\label{ineq:po}
	\end{align}
	If $A'_2 \cap A_1 = \emptyset$, then $A'_2 \subset A_2$ (since $A'_2 \neq A_2$). Hence, $v_2(A'_2)<v_2(A_2)$ so that $A'$ cannot Pareto dominate $A$. As such, we will assume that $A'_2 \cap A_1 \neq \emptyset$. Then, due to the ratio ordering and how $A_1, A_2$ are constructed, we must have
	\begin{align*}
	\frac{v_1(o)}{v_2(o)} &\ge \frac{v_1(o^*)}{v_2(o^*)} \quad \forall o \in A'_2 \cap A_1\\
	\Rightarrow \quad \frac{v_1(A'_2 \cap A_1)}{v_2(A'_2 \cap A_1)} = \frac{\sum_{o \in A'_2 \cap A_1} v_1(o)}{\sum_{o \in A'_2 \cap A_1} v_2(o)} &\ge \frac{v_1(o^*)}{v_2(o^*)}.
	\end{align*}
	By similar reasoning, it holds that
	\[\frac{v_1(o^*)}{v_2(o^*)} \ge \frac{v_1(A'_1 \cap A_2)}{v_2(A'_1 \cap A_2)}. \]
	Combining with Inequality~(\ref{ineq:po}), we get
	\begin{align*}
	&\frac{v_1(A'_2 \cap A_1)}{v_2(A'_2 \cap A_1)} \ge \frac{v_1(A'_1 \cap A_2)}{v_2(A'_1 \cap A_2)} > \frac{v_1(A'_2 \cap A_1)}{v_2(A'_1 \cap A_2)}\\
	\Rightarrow \quad &v_2(A'_1 \cap A_2) > v_2(A'_2 \cap A_1), \qquad \text{since $v_1(A'_2 \cap A_1) > 0$,}\\
	\Rightarrow \quad &v_2(A'_2 \cap A_2) + v_2(A'_1 \cap A_2) > v_2(A'_2 \cap A_2) + v_2(A'_2 \cap A_1)\\
	\Rightarrow \quad &v_2(A_2) > v_2(A'_2),
	\end{align*}
	which contradicts the necessary condition for $A'$ to Pareto dominate $A$: $v_2(A'_2) \ge v_2(A_2)$. Assuming $v_2(A'_2) > v_2(A_2)$ leads us to an analogous conclusion. Hence, $A$ must be Pareto optimal.
	\paragraph{Complexity: } The $m$ ratios can be sorted in $O(m \log m)$ time. Each new \textbf{while} loop condition can be checked in $O(1)$ time, so the total time taken by the \textbf{while} loop is $O(m)$. Hence, the algorithm runs in $O(m\log m)$ time. \newline
	
Let us now address the scenario where there are items of zero value to an agent. Of course, items valued at zero by both agents can be safely ignored. 
We will initialize the bundle $A_i$ with items valued positively by agent $i \in \{1,2\}$ only, i.e., $A^0_1 = \{o \in O: v_1(o)>0, v_2(o)=0\}$ and $A^0_2 = \{o \in O: v_2(o)>0, v_1(o)=0\}$. Then we run Algorithm~\ref{alg_waw} on the remaining items and use its output $(A_1,A_2)$ to augment the respective bundles. We will now show that the resulting allocation $(A^0_1 \cup A_1, A^0_2 \cup A_2)$ is $\wefi$ and $\po$.

By the argument in the previous part, there is an item $o' \in A_2 \subseteq A^0_2 \cup A_2$ such that 
\[\frac{v_1(A^0_1 \cup A_1)}{w_1} \ge \frac{v_1(A_1)}{w_1} \ge \frac{v_1(A_2)-v_1(o')}{w_2} =  \frac{v_1(A^0_2 \cup A_2)-v_1(o')}{w_2},\]
since $v_1(A^0_2)=0$. Thus, agent $1$ is weighted envy-free towards agent $2$ up to one item. An analogous argument shows that agent $2$ is also weighted envy-free towards agent $1$ up to one item.

Since a Pareto optimal allocation cannot be incomplete (because each item has a positive value to at least one agent), it suffices to show that the (complete) allocation under consideration is not Pareto dominated by any complete allocation. Again, any complete allocation can be obtained from $(A^0_1 \cup A_1, A^0_2 \cup A_2)$ by swapping items between agents. It is evident that any allocation in which an item $o \in A^0_1$ (resp., an item $o \in A^0_2$) belongs to agent $2$ (resp., agent $1$) is Pareto dominated by the allocation wherein this item is given to agent $1$ (resp., agent $2$), everything else remaining the same. Hence, it suffices to show that a Pareto improvement cannot be achieved by swapping items in $A_1 \cup A_2$ between the agents---but we already know this from the earlier part of the proof.
\end{proof}

\subsection{Any Number of Agents}\label{sec:wef1po_gen}
Having resolved the existence question of $\po$ and $\wefi$ for two agents, we now investigate whether such an allocation always exists for any number of agents, answering the question in the affirmative. 
To this end, we employ a weighted modification of the algorithm by \citet{barman2018finding}, which finds a $\po$ and $\efi$ allocation in pseudo-polynomial time for agents with additive valuations in the unweighted setting. 
Like Barman et al., we consider an artificial market where each item has a price and agents purchase a bundle of items with the highest ratio of value to price, called ``bang per buck ratio''. This allows us to measure the degree of fairness of a given allocation in terms of the prices. 

Formally, a {\em price vector} is an $m$-dimensional non-negative real vector $\bfp=(p_1,p_2,\ldots,p_m) \in \R^O_{\ge 0}$; we call $p_o$ the {\em price} of item $o \in O$, and write $p(X)=\sum_{o \in X}p_{o}$ for a set of items $X$. 
Let $A$ be an allocation and $\bfp$ be a price vector. 
For each $i\in N$, we call $p(A_i)$ the {\em spending} and $\frac{1}{w_i}p(A_i)$ the {\em weighted spending} of agent $i$. We now define a weighted version of the price envy-freeness up to one item (pEF1) notion introduced by \citet{barman2018finding}. 

\begin{definition}
Given an allocation $A$ and a price vector $\bfp$, we say that $A$ is {\em weighted price envy-free up to one item} ($\wpefi$) with respect to $\bfp$ if for any pair of agents $i,j$, either $A_j=\emptyset$ or $\frac{1}{w_i}p(A_i) \geq \frac{1}{w_j} \min_{o \in A_j}p(A_j \setminus \{o\})$. 
\end{definition}

The {\em bang per buck ratio} of item $o$ for agent $i$ is $\frac{v_i(o)}{p_o}$; we write the maximum bang per buck ratio for agent $i$ as $\alpha_i(\bfp)$. 
We refer to the items with maximum bang per buck ratio for $i$ as $i$'s {\em MBB items} and denote the set of such items by $\MBB_i(\bfp)$ for each $i \in N$. The following lemma is a straightforward adaptation of Lemma $4.1$ in \cite{barman2018finding} to our setting; it ensures that one can obtain the property of $\wefi$ by balancing among the spending of agents under the MBB condition. 

\begin{lemma}\label{lem:MBBwpEF1}
Given a complete allocation $A$ and a price vector $\bfp$, suppose that allocation $A$ satisfies $\wpefi$ with respect to $\bfp$ and agents are assigned to MBB items only, i.e., $A_i \subseteq \MBB_i(\bfp)$ for each $i \in N$. Then $A$ is $\wefi$.  
\end{lemma}

\begin{proof}
To show that $A$ is $\wefi$, take any pair of agents $i,j \in N$. If $A_j=\emptyset$, the required condition holds trivially. 
Suppose that $A_j\neq\emptyset$. 
Since the allocation $A$ is $\wpefi$ with respect to $\bfp$, 
$$
\frac{1}{w_i} p(A_i) \ge \frac{1}{w_j} \min_{o \in A_j} p(A_j \setminus \{o\}). 
$$
Multiplying both sides by $\alpha_i(\bfp)$, we obtain
\begin{align*}
&\frac{1}{w_i} \alpha_i(\bfp) \cdot p(A_i) \ge \frac{1}{w_j} \min_{o \in A_j} \alpha_i(\bfp) \cdot p(A_j \setminus \{o\})\\
&\Rightarrow \frac{1}{w_i} \sum_{o \in A_i} \frac{v_i(o)}{p_{o}} p_{o} \ge \frac{1}{w_j} \min_{o \in A_j} \sum_{o' \in A_j\setminus \{o\}} \frac{v_i(o')}{p_{o'}} p_{o'}\\
&\Rightarrow \frac{1}{w_i} \sum_{o \in A_i} v_i(o) \ge \frac{1}{w_j} \min_{o \in A_j} \sum_{o' \in A_j\setminus \{o\}} v_i(o')\\
&\Rightarrow \frac{1}{w_i} v_i(A_i) \ge \frac{1}{w_j} \min_{o \in A_j}  v_i(A_j \setminus \{o\}). 
\end{align*}
For the transition from the first to the second inequality in the chain, we use the definition of $\alpha_i(\bfp)$ and the assumption $A_i \subseteq \MBB_i(\bfp)$: by the definition of $\alpha_i(\bfp)$, we have $\alpha_i(\bfp) \ge \frac{v_i(o')}{p_{o'}}$ for all $o'\in O$, and by the assumption $A_i \subseteq \MBB_i(\bfp)$, it holds that $\alpha_i(\bfp) = \frac{v_i(o)}{p_{o}}$ for all $o\in A_i$.
The last inequality allows us to conclude that $A$ is $\wefi$. 
\end{proof}

It is also known that if each agent $i$ only purchases MBB items, so that $i$ maximizes her valuation under the budget $p(A_i)$, then the corresponding allocation is Pareto optimal. 

\begin{lemma}[First Welfare Theorem; \citet{Mas-Colell1995}, Chapter $16$]\label{lem:MBB}
Given a complete allocation $A$ and a price vector $\bfp$, suppose that agents are assigned to MBB items only, i.e., $A_i \subseteq \MBB_i(\bfp)$ for each $i \in N$. Then $A$ is $\po$. 
\end{lemma}

\begin{proof}
To show that $A$ is Pareto optimal, suppose towards a contradiction that another allocation $A'$ Pareto dominates $A$. 
This means that $v_i(A'_i) \ge v_i(A_i)$ for all $i \in N$, and $v_j(A'_j) > v_j(A_j)$ for some $j \in N$. 
Since each agent $i$ maximizes her valuation under the budget $p(A_i)$ in $A$, we have $p(A'_i) \ge p(A_i)$ for all $i \in N$, and $p(A'_j) > p(A_j)$ for some $j \in N$. Indeed, if $p(A_i)>p(A'_i)$ for some $i \in N$, this would mean that 
\[
v_i(A_i)= \alpha_i(\bfp) p(A_i) > \alpha_i(\bfp) p(A'_i)  \ge  \sum_{o  \in A'_i}  \frac{v_i(o)}{p_o}p_o =v_i(A'_i),  
\]
a contradiction. Thus, $p(A'_i) \ge p(A_i)$ for all $i \in N$. 
A similar argument shows that $p(A'_j) > p(A_j)$ for some $j\in N$. 
However, this implies that
\[
\sum_{o \in O}p_o \ge \sum_{i \in N}p(A'_i) > \sum_{i \in N}p(A_i)= \sum_{o \in O}p_o, 
\]
where the last equality holds since $A$ is a complete allocation. We thus obtain the desired contradiction. 
\end{proof}

With Lemmas \ref{lem:MBBwpEF1} and \ref{lem:MBB}, the problem of finding a $\po$ and $\wefi$ allocation reduces to that of finding an allocation and price vector pair satisfying the MBB condition and $\wpefi$. We show that there is an algorithm that finds such an outcome in pseudo-polynomial time. Our algorithm follows a similar approach as that of \citet{barman2018finding}; thus the proof of Theorem~\ref{thm:POwEF1} is deferred to Appendix~\ref{app:thmproofs}. 

\begin{theorem}\label{thm:POwEF1}
For any number of agents with additive valuations and arbitrary positive real weights, there exists a $\wefi$ and $\po$ allocation. Furthermore, such an allocation can be computed in time $poly(m,n,v_{max},w_{max})$ for any integer-valued inputs, where $v_{max}:=\max_{i\in N, o \in O}v_{i}(o)$ and $w_{max}:=\max_{i \in N} w_i$.
\end{theorem}

The outline of the algorithm is as follows. Our algorithm alternates between two phases: the first phase involves reallocating items from large to small \emph{spenders} (where the ``spending'' of an agent is defined as the ratio between the price for her bundle of items and her weight; see the formal definition in Appendix~\ref{app:thmproofs}), and the second phase involves increasing the prices of the items owned by small spenders. We show that by increasing prices gradually, the algorithm converges to an allocation and price vector pair satisfying the desired criteria when both input weights and valuations are expressed as integral powers of $(1+\epsilon)$ for some $\epsilon>0$. Similarly to \citet{barman2018finding}, we apply our algorithm to the $\epsilon$-approximate instance of the original input and show that for small enough $\epsilon$, the output of the algorithm satisfies the original MBB condition and $\wpefi$. We note that compared to \citet{barman2018finding}, the analysis becomes more involved due to the presence of weights. In particular, each price-rise phase takes into account not only the valuations but also the weights; as a result, $\epsilon$ needs to be much smaller in order to ensure the equivalence. 

\section{$\wwefi$ and $\po$ allocations: Maximum Weighted Nash Welfare}\label{sec:wwef1_po}

In the previous section, we saw that $\MNW$ allocations may fail to satisfy $\wefi$, showing that the result of \citet{caragiannis2016unreasonable} from the unweighted setting does not extend to the weighted setting via $\wefi$ (or even $\wefc$ for any fixed $c$). 
Given that these allocations maximize a natural objective, it is still tempting to ask whether they provide any fairness guarantee.
The answer is indeed positive: we show that a $\MNW$ allocation satisfies $\wwefi$, a weaker fairness notion that also generalizes $\efi$. 

\begin{theorem}\label{thm:mwnw}
	For any number of agents with additive valuations and arbitrary positive real weights, a $\MNW$ allocation is always $\wwefi$ and $\po$.
\end{theorem}

The proof of Theorem~\ref{thm:mwnw} follows a similar outline as the corresponding proof of \citet{caragiannis2016unreasonable}.
$\po$ follows easily from the definition of $\MNW$.
For $\wwefi$, we assume for contradiction that an agent~$i$ weakly envies another agent~$j$ up to more than one item in a $\MNW$ allocation.
If every agent has a positive value for every item, we pick an item in agent~$j$'s bundle for which the ratio between $i$'s value and $j$'s value is maximized.
By distinguishing between the cases $w_i\ge w_j$ and $w_i\le w_j$, we show that we can achieve a higher weighted Nash welfare upon transferring this item to agent~$i$'s bundle, which yields the desired contradiction.
The case where agents may have zero value for items is then handled separately.

\begin{proof}[Proof of Theorem~\ref{thm:mwnw}]
	Let $A$ be a $\MNW$ allocation, with $N_{\max}$ being the subset of agents having strictly positive realized valuations under $A$. If it were not $\po$, there would exist an allocation $\hat{A}$ such that $v_i(\hat{A}_i) > v_i(A_i)$ for some $i \in N$ and $v_j(\hat{A}_j) \ge v_j(A_j)$ for every $j \in N \setminus \{i\}$. If $i \in N\setminus N_{\max}$, we would have $v_j(\hat{A}_j) > 0$ for every $j \in N_{\max} \cup \{i\}$, contradicting the assumption that $N_{\max}$ is a largest subset of agents to whom it is possible to give positive value simultaneously. If $i \in N_{\max}$, then $\prod_{j \in N_{\max}} v_j(\hat{A}_j)^{w_j} > \prod_{j \in N_{\max}} v_j(A_j)^{w_j}$, which violates the optimality of the right-hand side. This proves that $A$ is $\po$.

	Like \citet{caragiannis2016unreasonable}, we will start by proving that $A$ is $\wwefi$ for the scenario $N_{\max}=N$ and then address the case $N_{\max}\neq N$. Assume that $N_{\max}=N$. If $A$ is not $\wwefi$, then there exists a pair of agents $i,j \in N$ such that $i$ has weak weighted envy towards $j$ up to more than one item. Clearly, there must be at least two items in $j$'s bundle that $i$ values positively. 
	Moreover, $j$ must value these items positively as well---otherwise we can transfer them to $i$ and obtain a Pareto improvement.
	
	Let $A_j^i := \{o \in A_j: v_i(o)>0\}$. We construct another allocation $A'$ by transferring an item $o^*$ (to be chosen later) from $j$ to $i$ so that $A'_i=A_i \cup \{o^*\}$, $A'_j = A_j \setminus \{o^*\}$, and $A'_r=A_r$ for all $r \in N \setminus \{i,j\}$. We have
	\begin{align*}
	\frac{\NW(A')}{\NW(A)} &= \left( \frac{v_i(A_i \cup \{o^*\})}{v_i(A_i)}\right)^{w_i} \left( \frac{v_j(A_j \setminus \{o^*\})}{v_j(A_j)}\right)^{w_j}\\
	&=\left( \frac{v_i(A_i)+v_i(o^*)}{v_i(A_i)}\right)^{w_i} \left( \frac{v_j(A_j)-v_j(o^*)}{v_j(A_j)}\right)^{w_j}\\
	&=\left( 1+\frac{v_i(o^*)}{v_i(A_i)}\right)^{w_i} \left( 1-\frac{v_j(o^*)}{v_j(A_j)}\right)^{w_j}.
	\end{align*}
First, note that $v_j(o) > 0$ for all $o \in A_j^i$; otherwise the above ratio for $o^*\in A_j^i$ with $v_j(o^*)=0$ equals $\left( 1+\frac{v_i(o^*)}{v_i(A_i)}\right)^{w_i} > 1$, contradicting the assumption that $A$ is a $\MNW$ allocation. However, even under this condition, we will show that if agents $i,j$ violated the $\wwefi$ property, the above ratio would still exceed $1$ for some item $o^*$.  	
\paragraph{Case I} $w_i \ge w_j$. 

Let us pick an item $o^* \in \arg\min_{o \in A_j^i} \frac{v_j(o)}{v_i(o)}$ specifically to transfer from $j$ to $i$ for changing the allocation from $A$ to $A'$. This is well-defined by the definition of $A_j^i$. Consider
	\begin{align*}
	\left[\frac{\NW(A')}{\NW(A)} \right]^\frac{1}{w_j} &= \left( 1+\frac{v_i(o^*)}{v_i(A_i)}\right)^\frac{w_i}{w_j} \left( 1-\frac{v_j(o^*)}{v_j(A_j)}\right),
	\end{align*}
where $1- \frac{v_j(o^*)}{v_j(A_j)}>0$ since $v_j(A_j) > v_j(o^*) > 0$.
Moreover, we have
\[\left( 1+\frac{v_i(o^*)}{v_i(A_i)}\right)^\frac{w_i}{w_j}  \ge \left( 1+\frac{w_i}{w_j} \cdot \frac{v_i(o^*)}{v_i(A_i)}\right)\]
from Bernoulli's inequality, since $\frac{v_i(o^*)}{v_i(A_i)}>0$ and $\frac{w_i}{w_j} \geq 1$. 
Algebraic manipulations show that
\begin{align}
&\left( 1+\frac{w_i}{w_j} \cdot \frac{v_i(o^*)}{v_i(A_i)}\right) \left( 1- \frac{v_j(o^*)}{v_j(A_j)}\right) > 1 \notag \\
\Leftrightarrow \quad & \frac{v_i(A_i)}{w_i} < \frac{v_i(o^*)}{v_j(o^*)} \left( \frac{v_j(A_j)-v_j(o^*)}{w_j} \right). \label{case1ineq}
\end{align}
The latter inequality is true under our assumptions for the following reasons: Since agent $i$ with a larger weight has weak weighted envy towards agent $j$ with a smaller weight up to more than one item, we have $\frac{v_i(A_i)}{w_i} < \frac{v_i(A_j)-v_i(o^*)}{w_j}$; in addition, due to our choice of $o^*$,
\[\frac{v_j(o^*)}{v_i(o^*)} \le \frac{\sum_{o \in A_j^i}v_j(o)}{\sum_{o \in A_j^i}v_i(o)} \le \frac{\sum_{o \in A_j}v_j(o)}{\sum_{o \in A_j}v_i(o)} = \frac{v_j(A_j)}{v_i(A_j)}, \]
since $\sum_{o \in A_j \setminus A_j^i}v_j(o) \ge 0$ and $\sum_{o \in A_j \setminus A_j^i} v_i(o) = 0$. 
Plugging $v_i(A_j) \le \frac{v_i(o^*)}{v_j(o^*)}\cdot v_j(A_j)$ into the above strict inequality and simplifying, we obtain \eqref{case1ineq}. But chaining all these inequalities together, we get 
\[\left[\frac{\NW(A')}{\NW(A)} \right]^\frac{1}{w_j} > 1 \quad \Rightarrow \quad \NW(A')>\NW(A).\]
This is a contradiction, which shows that $A$ is $\wwefi$ in this case.
\paragraph{Case II} $w_i < w_j$.

As in Case I, we pick an item $o^* \in \arg\min_{o \in A_j^i} \frac{v_j(o)}{v_i(o)}$, so we also have $v_i(A_j) \le \frac{v_i(o^*)}{v_j(o^*)} \cdot v_j(A_j)$. 
Consider 
\begin{align*}
\left[\frac{\NW(A')}{\NW(A)} \right]^\frac{1}{w_i} &= \left( 1+\frac{v_i(o^*)}{v_i(A_i)}\right) \left( 1-\frac{v_j(o^*)}{v_j(A_j)}\right)^\frac{w_j}{w_i},
\end{align*}
where $1- \frac{v_j(o^*)}{v_j(A_j)}>0$ since $v_j(A_j) > v_j(o^*) > 0$.
Moreover, we have
\[\left(1-\frac{v_j(o^*)}{v_j(A_j)}\right)^\frac{w_j}{w_i} \ge \left( 1- \frac{w_j}{w_i} \cdot \frac{v_j(o^*)}{v_j(A_j)}\right)\]
from Bernoulli's inequality, since $\frac{w_j}{w_i} > 1$ and $0 < \frac{v_j(o^*)}{v_j(A_j)} < 1$. 
Algebraic manipulations show that
\begin{align}
&\left( 1+ \frac{v_i(o^*)}{v_i(A_i)}\right) \left( 1- \frac{w_j}{w_i} \cdot \frac{v_j(o^*)}{v_j(A_j)}\right) > 1 \notag \\
\Leftrightarrow \quad & \frac{v_i(A_i)+v_i(o^*)}{w_i} < \frac{v_i(o^*)}{v_j(o^*)} \cdot \frac{v_j(A_j)}{w_j}. \label{case2ineq}
\end{align}
Since agent $i$ with a smaller weight has weak weighted envy towards agent $j$ with a larger weight up to more than one item, we have $\frac{v_i(A_i)+v_i(o^*)}{w_i} < \frac{v_i(A_j)}{w_j}$.
Plugging $v_i(A_j) \le \frac{v_i(o^*)}{v_j(o^*)} \cdot v_j(A_j)$ into this strict inequality, we obtain \eqref{case2ineq}, and chaining all inequalities leads us to the contradiction $\NW(A')>\NW(A)$.
This completes the proof for the scenario where $N_{\max}=N$. 

The rest of the proof mirrors the corresponding part in the proof of \citet{caragiannis2016unreasonable}. If $N_{\max} \subsetneq N$, it is easy to see that there can be no weighted envy towards any $i \not\in N_{\max}$ since $v_j(A_i)=0$ for any such $i$ and every $j \in N$. Also, for any $i,j \in N_{\max}$, we can show as in the proof for $N_{\max}=N$ above that there cannot be (weak) weighted envy up to more than one item. Suppose for contradiction that an agent $i \in N \setminus N_{\max}$ is not weighted envy-free towards some $j \in N_{\max}$ up to one item under $A$. This means that $i$ still has positive value for $j$'s bundle even after removing any single item; in particular, $i$ values at least two items in $A_j$ positively. Since $j$ must also value these items positively, we may transfer one of them to $i$ and keep both $i$ and $j$'s valuations positive. This contradicts the maximality of $N_{\max}$. Hence, $i \in N\setminus N_{\max}$ must be weakly weighted envy-free up to one item towards $j \in N_{\max}$.
It follows that $A$ is $\wwefi$ in all cases.
\end{proof}

\section{$\wefi$ and other fairness notions}\label{sec:wef1_propMMS}

An allocation that satisfies multiple fairness guarantees is naturally desirable but often elusive in the setting with indivisible items. Hence, we will now explore the implications of the $\wefi$ property on the other fairness criteria defined in Section~\ref{sec:prelims}. 

For additive valuations, \citet{aziz2019polynomial} provided a polynomial-time algorithm for computing a $\po$ and $\WPi$ allocation,  whereas we proved the existence of $\po$ and $\wefi$ allocations in Section~\ref{sec:wef1po_gen}. 
It is straightforward to verify that, in the unweighted scenario, any complete envy-free allocation is also proportional for subadditive valuations and any complete $\efi$ allocation is $\propi$ for additive valuations (see, e.g., \cite{aziz2019polynomial}). This begs the question: does the $\wefi$ property along with completeness also imply the $\WPi$ condition? Unfortunately, the answer is no for any $n\ge 3$---in fact, we establish a much stronger result in the following proposition.
\begin{proposition}\label{prop:wprop}
	For any number $n \ge 2$ of agents with additive valuations and arbitrary positive real weights, any complete $\wefi$ allocation is $\WP(n-1)$.
	However, for any $n\ge 3$, there exists an instance in which no complete $\wefi$ allocation is $\WP(n-2)$.
\end{proposition}

\begin{proof}
	Let $A$ be a complete $\wefi$ allocation under additive valuations, and fix an agent $i$.
	By definition of $\wefi$, for each $j\in N\setminus\{i\}$, there is a set $S_j\subseteq A_j$ with $|S_j|\leq 1$ such that $\frac{w_j}{w_i} \cdot v_i(A_i) \ge v_i(A_j) - v_i(S_j)$.
	Summing up the respective inequalities for all $j\in N \setminus \{i\}$ and noting that $O = \cup_{j \in N} A_j$ due to completeness, we have
	$$\frac{\sum_{j \in N \setminus \{i\}} w_j}{w_i} \cdot v_i(A_i) \ge v_i(O) - v_i(A_i) - \sum_{j \in N \setminus \{i\}} v_i(S_j).$$
	If we let $S:=\cup_{j \in N \setminus \{i\}}S_j$, the above inequality implies that
	$\frac{\sum_{j \in N} w_j}{w_i} \cdot v_i(A_i) \ge v_i(O) - v_i(S)$.
	It follows that
	$$v_i(A_i)\ge \frac{w_i}{\sum_{j \in N} w_j}\cdot v_i(O) - \frac{w_i}{\sum_{j \in N} w_j}\cdot v_i(S) 
	\ge 
	\frac{w_i}{\sum_{j \in N} w_j}\cdot v_i(O) - v_i(S),$$
	where the latter inequality follows from $\frac{w_i}{\sum_{j \in N} w_j} < 1$ and $v_i(S)\ge 0$.
	Since $|S|\leq n-1$, this shows that $A$ is $\WP(n-1)$.
	
	For the second part, let $0 < \epsilon < \frac{1}{n(n-1)}$, and consider an instance with $n \ge 3$ agents where $w_1=1-(n-1)\epsilon$ and $w_i = \epsilon$ for all $i \in N\setminus\{1\}$.
	Moreover, assume that there are $m=n$ items, $v_1(o)=\frac{1}{n}$  for every $o \in O$, and every other agent has an arbitrary positive value for each item.
	Let $A$ be any complete $\wefi$ allocation.
	Note that every agent must receive exactly one item in $A$---otherwise there would be an agent with no item and another agent with two or more items, and these two agents would violate the $\wefi$ property.
	Then, for any set $S\subseteq O\setminus A_1$ with $|S|\leq n-2$, we have 
	\[\frac{w_1}{\sum_{j \in N} w_j}v_1(O) - v_1(S) \ge \frac{1-(n-1)\epsilon}{1} \cdot 1 - \frac{n-2}{n} = \frac{2}{n} - (n-1)\epsilon > \frac{2}{n} - \frac{1}{n} = \frac{1}{n} = v_1(A_1).\]
    Hence, $A$ is not $\WP(n-2)$.
\end{proof}

For $n$ symmetric (unweighted) agents with additive valuations, \citet[Prop. 3.6]{amanatidis18comparing} showed that any complete $\efi$ allocation is $\frac{1}{n}$-$\mathtt{MMS}$, and this approximation guarantee is tight. 
Moreover, as \citet[Thm. 4.1]{caragiannis2016unreasonable} proved, a maximum Nash welfare allocation, which is $\efi$ and $\po$, is also $\Theta(1/\sqrt{n})$-$\mathtt{MMS}$. This means that, for a small number of agents, the $\efi$ property provides a reasonable approximation to $\mathtt{MMS}$ fairness. 
Our next proposition stands in stark contrast to these results: For any number of agents with asymmetric weights, it shows that the $\wefi$ condition does not imply any positive approximation of the $\WM$ guarantee, even in conjunction with Pareto optimality. 
\begin{proposition}\label{prop:wef1_not_apprwmms}
For any constant $\epsilon > 0$ and any number $n\geq 2$ of agents, there exists an instance with additive valuations in which some $\po$ and $\wefi$ allocation is not $\epsilon$-$\WM$.
\end{proposition}

\begin{proof}
	Suppose $m=n$ and the weights are $w_1 > w_2 > \dots > w_n > 0$ with $w_n/w_1 < \epsilon$.
	The valuation functions are defined by
	\begin{align*}
		v_1(o_r) &= w_r &\forall r \in [n];\\
		v_i(o_r) &=\begin{cases}
		1 &\text{if $r=i-1$}\\
		0 &\text{otherwise}
		\end{cases} &\forall i \in N \setminus \{1\}.
	\end{align*}
	 
	 One can check that $\WM_1 = w_1$, obtained from the allocation that gives item $o_i$ to agent $i$ for every $i \in N$. Moreover, the allocation $A$ in which agent $1$ receives $o_n$ and agent $i$ receives $o_{i-1}$ for every $i \in N \setminus \{1\}$ is $\po$ and $\wefi$. The $\WM$ approximation ratio for agent $1$ under $A$ is $v_1(o_n)/\WM_1=w_n/w_1 < \epsilon$.
\end{proof}

For the special case of $n=2$ agents, the first part of Proposition~\ref{prop:wprop} implies that the output of the weighted adjusted winner procedure from Section~\ref{sec:wef1_po_twoagents} is always $\WPi$.
However, we show that it comes with no guarantee on the $\WM$ approximation.
\begin{proposition}\label{prop:weighted-adjusted-bad}
	The output of the weighted adjusted winner procedure (Algorithm~\ref{alg_waw}) is always $\WPi$. However, for any constant $\epsilon > 0$, it is not necessarily $\epsilon$-$\WM$.
\end{proposition}

\begin{proof}
		By Theorem~\ref{thm:2agents}, Algorithm~\ref{alg_waw} produces a complete $\wefi$ allocation for two agents; such an allocation must be $\WP1$ due to Proposition~\ref{prop:wprop}.
		
		Consider an instance with $m=n=2$ and weights $w_1 > w_2 > 0$ with $w_2/w_1 < \epsilon$.
		Let $\delta > 0$ be such that $w_2/w_1 > \delta/(1-\delta)$.
		Suppose that $v_1(o_i)=w_i$ for $i\in\{1,2\}$, $v_2(o_1) = 1-\delta$, and $v_2(o_2) = \delta$.
		Since
		\[\frac{v_1(o_2)}{v_2(o_2)} = \frac{w_2}{\delta}  > \frac{w_1}{1-\delta}=\frac{v_1(o_1)}{v_2(o_1)},\]
		Algorithm~\ref{alg_waw} considers the items in the order $o_2,o_1$.
		Then agent 1 receives $o_2$ and agent 2 receives $o_1$.
		Hence, agent $1$'s $\WM$ approximation ratio is $w_2/w_1<\epsilon$. 
\end{proof}

\citet{farhadi2019fair} showed that a $1/n$ approximation of the $\WM$ can be obtained using the canonical (unweighted) round-robin algorithm with the added requirement that agents with higher weights go first; it is not hard to check that this algorithm does not guarantee $\wwefc$ or $\WPc$ for any constant~$c$.
On the other hand, while our weighted picking sequence from Section~\ref{sec:picking} ensures $\wefi$, perhaps surprisingly, it does not come with any $\WM$ approximation guarantee.

\begin{proposition}
For any constant $\epsilon > 0$, the output of the Weighted Picking Sequence protocol (Algorithm~\ref{alg_pickseq}) is not necessarily $\epsilon$-$\WM$.
\end{proposition}

\begin{proof}
Consider an instance with $n=2$ and $m=k+1$ for a positive integer $k > 1/\epsilon$.
The weights are $w_1=k$ and $w_2=1$, and both agents have identical valuations with value $k^2$ for one item and $1$ for each of the remaining $k$ items.
We have $\WM_1 = k^2$ and $\WM_2 = k$.
The picking sequence induced by the algorithm is $1,2,1,1,\dots,1$, so agent 2 only receives value $1$, which is $1/k < \epsilon$ of her $\WM$.  
In fact, even if we let agent 2 pick first, the picking sequence will be $2,1,1,\dots,1$.
In this case, agent 1 only receives value $k$, which is again $1/k< \epsilon$ of her $\WM$.
\end{proof}

\section{Experiments}\label{sec:expts}

\begin{figure}[!ht]
	\begin{center}
		\begin{tabular}{cc}
			\includegraphics[width=0.48\columnwidth]{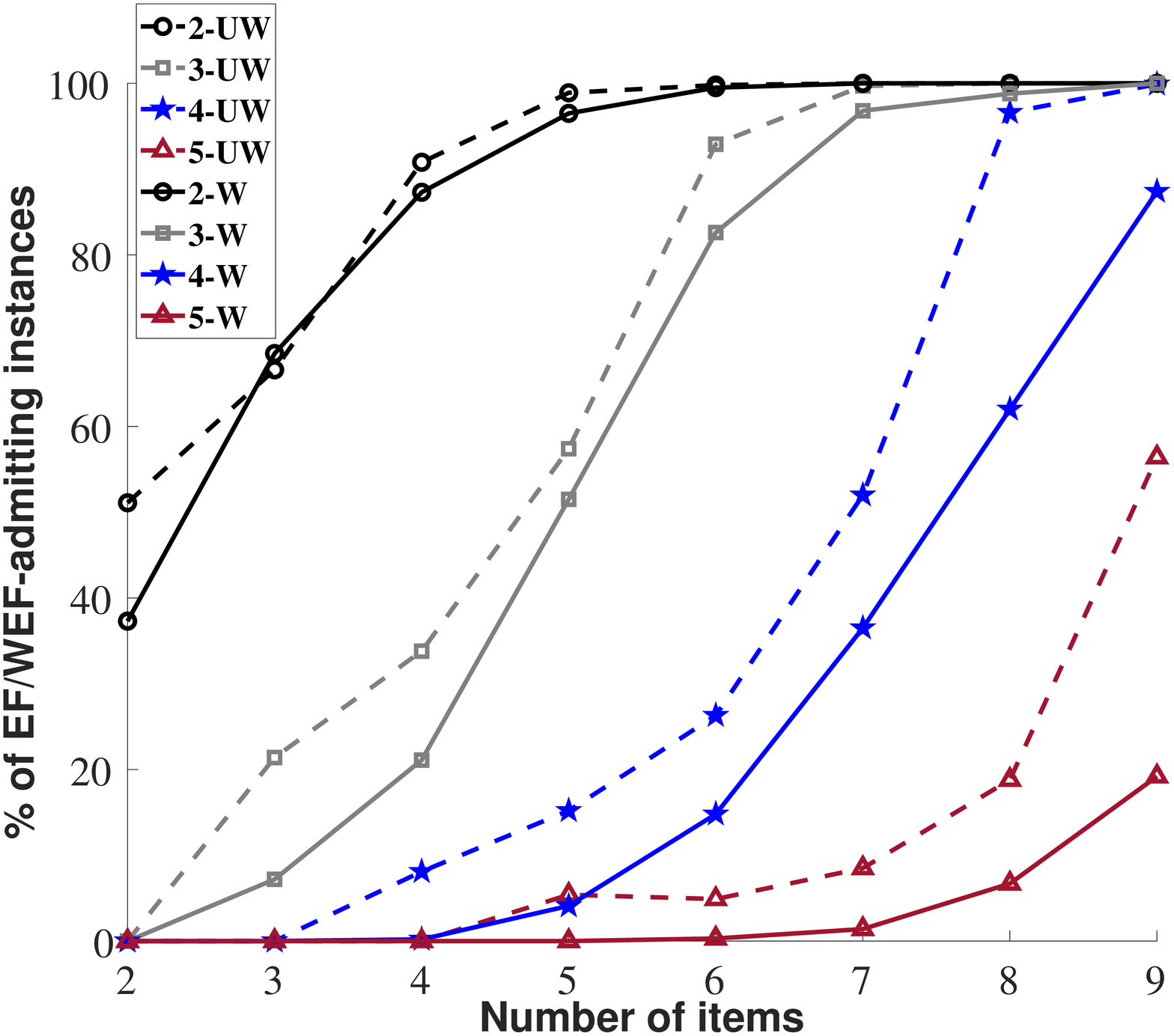} & \includegraphics[width=0.48\columnwidth]{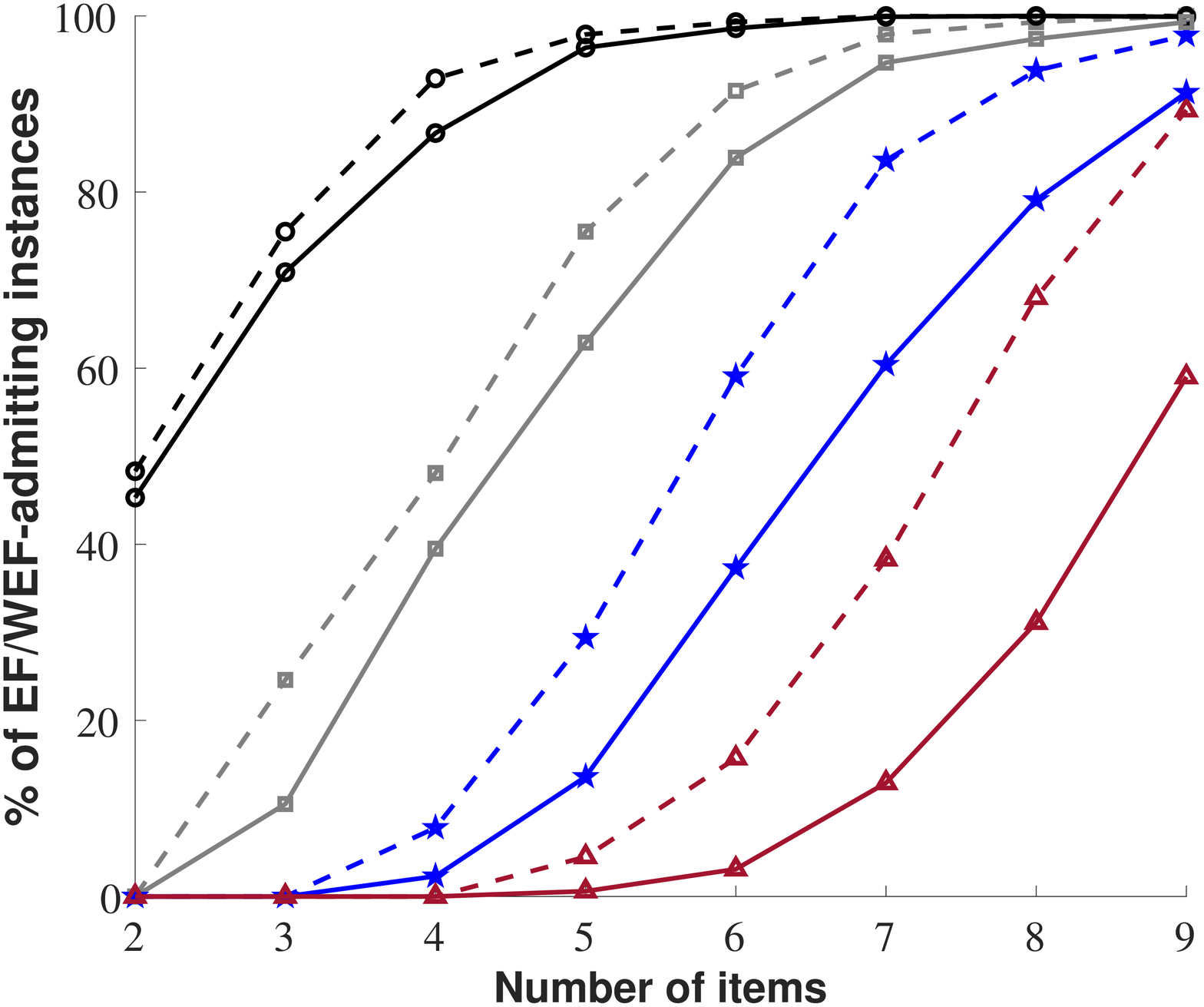}\\ 
			(a) Uniform distribution on $[0,1]$ & (b) Exponential distribution with mean $1$\\
			\includegraphics[width=0.48\columnwidth]{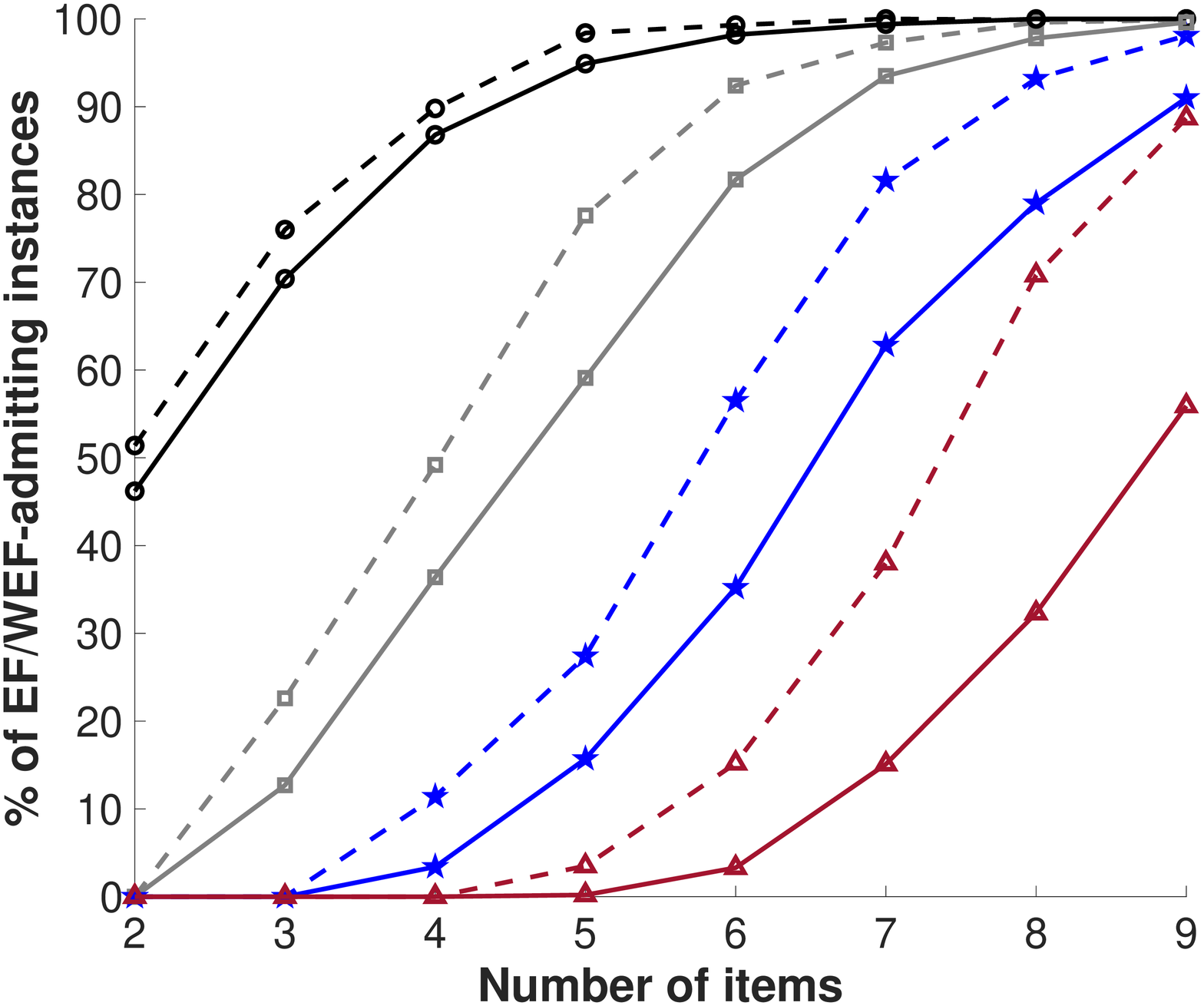} & \includegraphics[width=0.48\columnwidth]{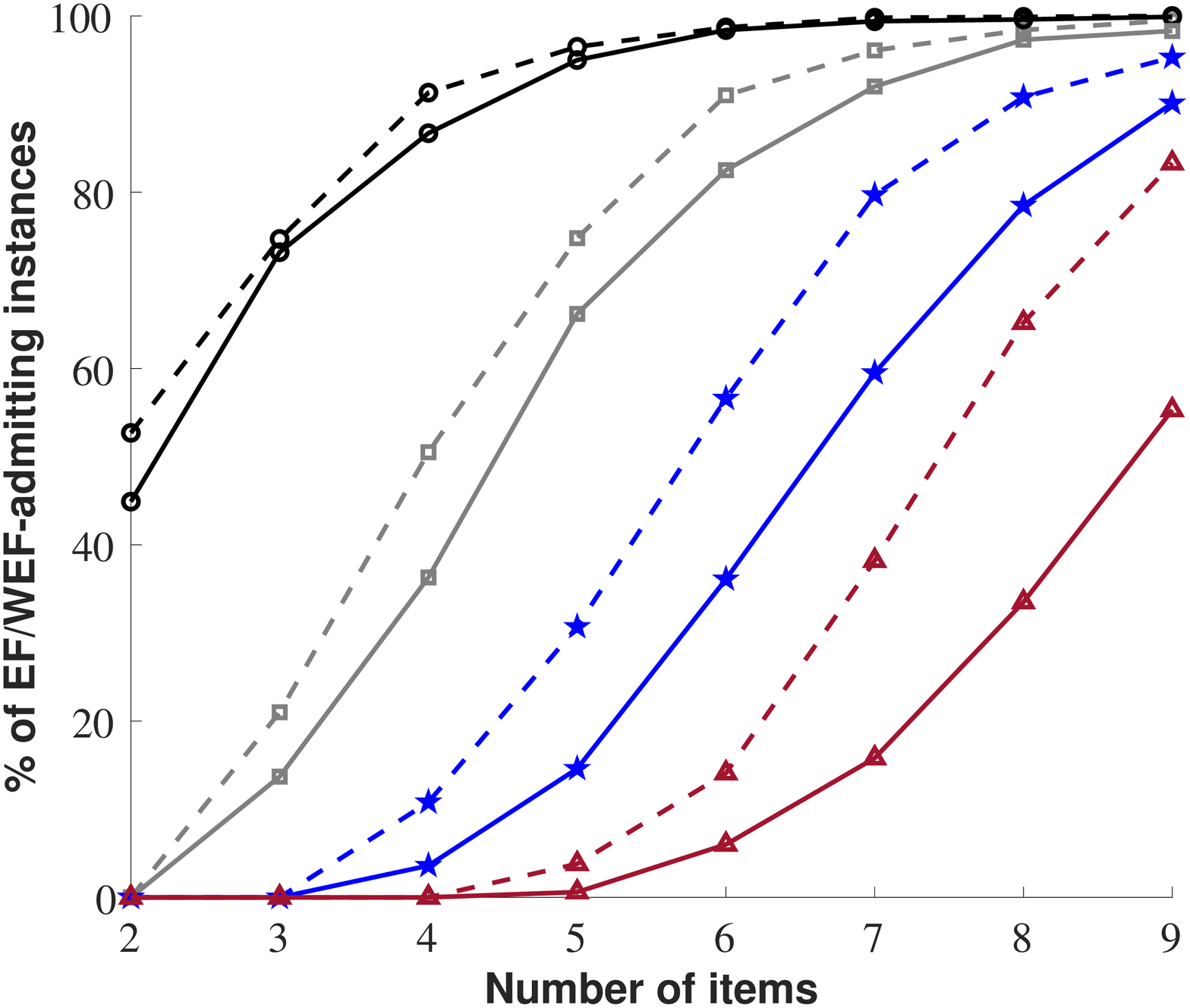}\\ 
			(c) Exponential distribution with mean $2$ & (d) Log-normal dist. with parameters $(0,1)$\\
			\includegraphics[width=0.48\columnwidth]{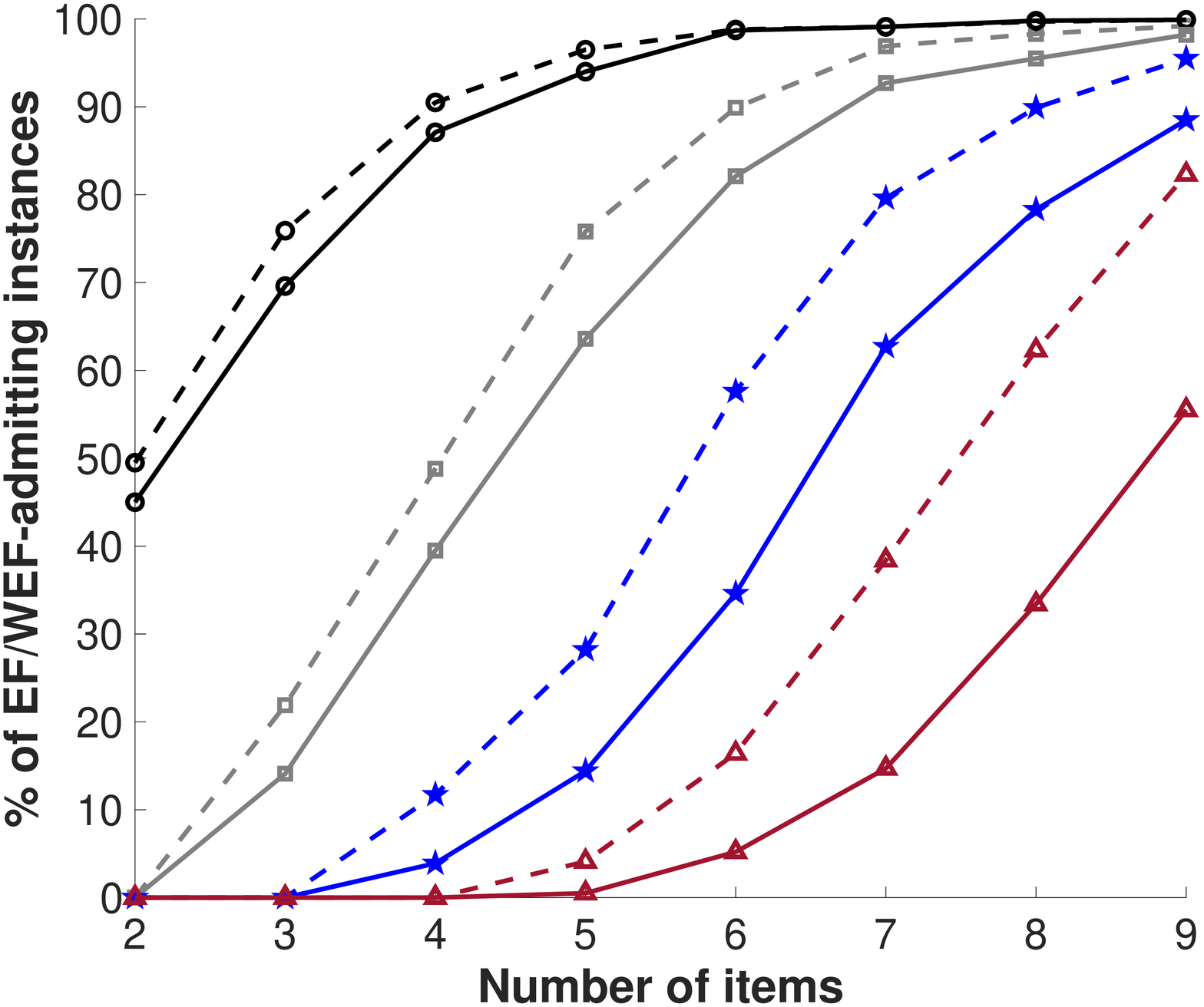} & \includegraphics[width=0.48\columnwidth]{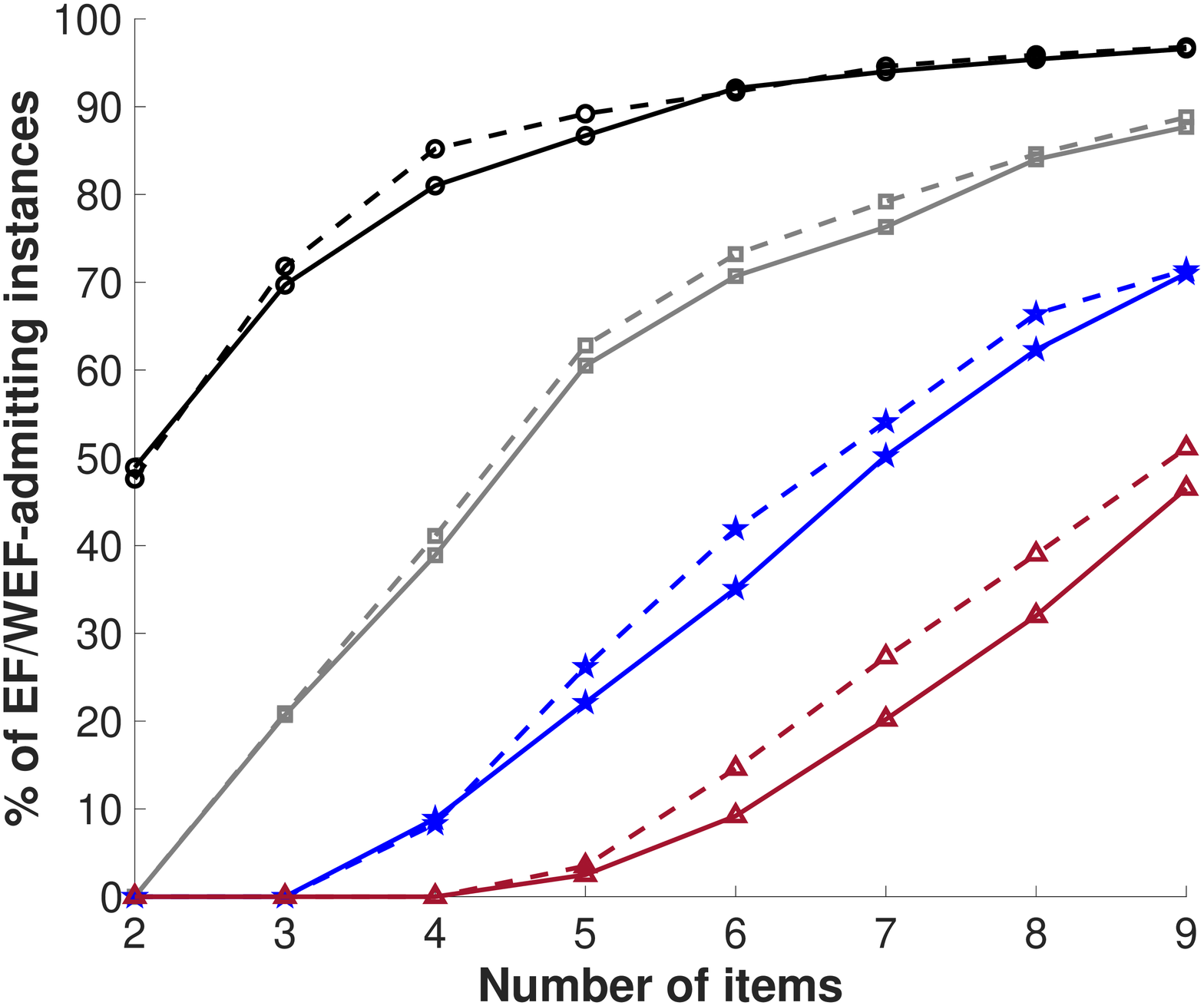}\\ 
			(e) Log-normal dist. with parameters $(1,1)$ & (f) Log-normal dist. with parameters $(0,2)$
		\end{tabular}
	\end{center}
	\caption{Percentages of instances that admit $\mathtt{EF}$ and $\wef$ allocations for different valuation distributions in our experiments; $n$-UW, depicted by dashed curves (resp., $n$-W, depicted by solid curves) refers to a scenario with $n$ agents with equal weights (resp., weights proportional to agent indices) for all figures. \label{fig:wef_expts_full}}
\end{figure}

Thus far, we have extensively investigated the existence and computational properties of approximations to \emph{weighted envy-freeness} ($\wef$). While the $\wef$ notion itself cannot always be satisfied with indivisible items, an interesting question is how ``likely'' it is for a problem instance with weighted agents to admit a $\wef$ allocation, and how the results compare to those for envy-freeness in the unweighted setting. 

In this section, we approach this question experimentally by generating sets of $1000$ instances with $n \in \{2,3,4,5\}$ agents and $m \in \{2,3,\dots,9\}$ items wherein each agent's value for each item is drawn independently from a distribution. 
We perform our experiments on three common families of distributions---uniform, exponential, and log-normal---and consider two different weight vectors: $w_i=1$ (unweighted) and $w_i=i$ (weighted) for every $i \in N$.
For each generated instance, we determine by exhaustive search over all allocations whether a $\wef$ allocation exists.
The results are shown in Figure~\ref{fig:wef_expts_full}; the main observations are as follows:
\begin{enumerate}
\item For each fixed distribution and number of agents/items, weighted envy-free allocations are almost always harder to find than their unweighted analogs. 
Intuitively, we need more items to satisfy agents with larger weights, and this is not sufficiently compensated by the items that we can save through agents with smaller weights.
\item The more items we have, the more likely it is that a fair allocation exists. 
This is to be expected, since more items means higher flexibility in choosing the allocation.
\item By contrast, the more agents there are, the less likely it is that a fair allocation exists. This is again reasonable, as we need to satisfy a larger number of preferences when there are more agents.
\item Fair allocations are rarer when the distribution is uniform than when it is skewed.
This observation aligns with the intuition that in a uniform distribution the values are more evenly distributed, so it is harder to find items for which one agent has high value whereas the remaining agents have low value.
\end{enumerate}
Our experimental results illustrate the difficulty of achieving weighted envy-freeness and further justify our quest for the (strong and weak) relaxations of the $\wef$ property.

\section{Discussion and Future Work}\label{sec:nonadd}
In this article, we have introduced and studied envy-based notions for the allocation of indivisible items in a general setting where agents can have different entitlements.
As most of our results hold for additive valuation functions, the reader may wonder whether they can be extended to more general classes---after all, in the absence of weights, an $\efi$ allocation is known to exist for arbitrary monotone valuations \cite{lipton2004approximately}.
We therefore point out some hurdles that we faced while trying to generalize our weighted envy concepts beyond additive valuations.
First, we show that even for simple non-additive valuations, the existence of a $\wefi$ or $\wwefi$ allocation can no longer be guaranteed.
Since $\wwefi$ is weaker than $\wefi$, it suffices to prove the claim for $\wwefi$.

\begin{proposition}\label{mwnw_notef1}
	There exists an instance with $n=2$ agents such that one of the agents has a (normalized and monotone) submodular valuation,\footnote{A valuation function $v:2^O\rightarrow\mathbb{R}_{\ge 0}$ is said to be \emph{submodular} if for any $O_1\subseteq O_2\subseteq O$ and any item $o\in O\setminus O_2$, we have $v(O_1\cup\{o\})-v(O_1) \geq v(O_2\cup\{o\})-v(O_2)$.} the other agent has an additive valuation, and a complete $\wwefi$ allocation does not exist.
\end{proposition}

\begin{proof}

Consider an instance with two agents who have weights $w_1=1$ and $w_2=2$, and suppose that there are $m > 5$ items.
The valuation functions are given by $v_1(S)=|S|$ and $v_2(S)=1$ for every $S \in 2^O \setminus \emptyset$, and  $v_1(\emptyset)=v_2(\emptyset)=0$.
The functions are normalized and monotone, with $v_1$ additive and $v_2$ submodular.

For any complete allocation $A$, if $|A_1|\ge 2$, we have $v_2(A_1\setminus \{o\})/w_1=1 > 1/2 \ge v_2(A_2)/w_2$ for every $o \in A_1$.
Moreover, it holds that $v_2(A_2\cup\{o\})/w_2 = 1/2 < 1 = v_2(A_1)/w_1$ for every $o\in A_1$. Thus, the only way to make agent $2$ weakly weighted envy-free up to one item towards agent $1$ is to ensure that $|A_1|\leq 1$. Assume without loss of generality that $A_1=\{o_1\}$ (if $A_1=\emptyset$, agent $1$ will be even worse off in the argument that follows), so $A_2=O\setminus \{o_1\}$.
We have $v_1(A_1)=1$ and $v_1(A_2)=|A_2|=m-1$. 
Since agent $1$ has an additive valuation and a smaller weight than agent $2$, she would be weakly weighted envy-free up to one item towards agent $2$ if and only if there is an item $o \in A_2$ such that $v_1(A_1 \cup \{o\})/w_1 \ge v_1(A_2)/w_2$.
However, for any $o \in A_2$, we have $v_1(A_1 \cup \{o\})/w_1 = v_1(\{o_1,o\})=2$, whereas $v_1(A_2)/w_2 = (m-1)/2 > 2$ since $m>5$.
This means that no complete allocation can be $\wwefi$.
\end{proof}

By increasing the lower bound on the number of items in the instance of this proof to $5c$, one can show that a complete $\wwefc$ allocation is also not guaranteed to exist for any constant $c$.

One of the key ideas in our analysis of the maximum weighted Nash welfare allocation (Theorem~\ref{thm:mwnw}) is what we call the \emph{transferability} property: If agent $i$ has weighted envy towards agent $j$ under additive valuations, then there is at least one item $o$ in $j$'s bundle for which agent $i$ has positive (marginal) valuation---in other words, the item $o$ could be transferred from $j$ to $i$ to augment $i$'s realized valuation.\footnote{Transferability and related properties have been studied by \citet{babaioff2020fair} and \citet{benabbou2020finding} in the context of $\efi$ and $\po$ allocations for a subclass of submodular valuations.} Unfortunately, this property no longer holds for non-additive valuations. 

\begin{proposition}\label{non_transfer}
	There exists an instance such that an agent $i$ with a non-additive valuation function has weighted envy towards an agent $j$ under some allocation $A$, but there is no item in $j$'s bundle for which $i$ has positive marginal valuation---i.e., $\not\exists o \in A_j$ such that $v_i(A_i \cup \{o\}) > v_i(A_i)$. 
\end{proposition} 
\begin{proof}
	Consider the example in Proposition~\ref{mwnw_notef1}. Under any allocation with $|A_1|=m-1$ and $|A_2|=1$, agent $2$ has weighted envy towards agent $1$ since $v_2(A_2)=1/2<1=v_2(A_1)/w_1$. However, $v_2(A_2 \cup \{o\}) = 1 = v_2(A_2)$ for every $o \in A_1$.
\end{proof}

In light of these negative results, an important direction for future research is to identify appropriate weighted envy notions for non-additive valuations.
Other interesting directions include establishing conditions under which $\wef$ allocations are likely to exist (cf. Section~\ref{sec:expts}),\footnote{This has been done in the unweighted setting \citep{dickerson2014computational,manurangsi2019when,manurangsi2020closing}} investigating weighted envy in the allocation of \emph{chores} (items with negative valuations) \citep{aziz2019weighted} or combinations of goods and chores \citep{AzizCaIg19,aziz2019polynomial}, incorporating connectivity constraints \citep{BouveretCeEl17,bilo2019almost,BeiIgLu21}, and considering weighted versions of other envy-freeness approximations such as \emph{envy-freeness up to any item (EFX)} \citep{caragiannis2016unreasonable,plaut2018almost}.
From a broader point of view, our work demonstrates that fair division with different entitlements is richer and more challenging than its traditional counterpart in several ways, and much interesting work remains to be done.

\section*{Acknowledgments}
This work was partly done while Chakraborty and Zick were at the National University of Singapore, Igarashi was at the University of Tokyo, and Suksompong was at the University of Oxford.
Chakraborty and Zick were supported by the Singapore NRF Research Fellowship R-252-000-750-733. Igarashi was supported by the KAKENHI Grant-in-Aid for JSPS Fellows
no. 18J00997 and JST, ACT-X. 
Suksompong was supported by the European Research Council (ERC) under grant no.
639945 (ACCORD); his visit to NUS was supported by MOE Grant R-252-000-625-133.
Preliminary versions of the article appeared in Proceedings of the 19th International Conference on Autonomous Agents and Multiagent Systems, May 2020, and Proceedings of the 2nd Games, Agents, and Incentives Workshop, May 2020.
We would like to thank the associate editor and the anonymous reviewers for several helpful comments.

\bibliographystyle{ACM-Reference-Format}
\bibliography{abb,main}

\appendix

\section{Proof of Theorem \ref{thm:POwEF1}}\label{app:thmproofs}

\subsection{Preliminaries}
Before we proceed to present the algorithm formally, we introduce the following definitions and notations. We write the weighted version of the least spending as 
\[
\wmin(A,\bfp):=\min_{i \in N} \frac{1}{w_i} p(A_i). 
\]
We call $i \in N$ with $\frac{1}{w_i}p(A_i)=\wmin(A,\bfp)$ a {\em least weighted spender}. Also, we say that $j$ is a {\em violator} if some agent weighted-price-envies $j$ by more than one item, i.e., $A_j \neq \emptyset$ and
$$
\min_{o \in A_j} \frac{1}{w_j} p(A_j \setminus \{o\}) > \wmin(A,\bfp).
$$
Given $\epsilon \in (0,1)$, we say that $j$ is an {\em $\epsilon$-violator} if $A_j \neq \emptyset$ and
$$
\min_{o \in A_j} \frac{1}{w_j} p(A_j \setminus \{o\}) > (1+\epsilon)\wmin(A,\bfp).
$$
We write the maximum value of the left-hand side as 
\[
\wmax(A,\bfp):=\max_{j \in N: A_j \neq \emptyset} \min_{o \in A_j} \frac{1}{w_j} p(A_j \setminus \{o\}). 
\]
The pair $(A,\bfp)$ is said to be {\em $\epsilon$-weighted price EF1} ($\epsilon$-$\wpefi$) if no agent is an $\epsilon$-violator. 

Now, we define the {\em maximum bang per buck (MBB) network} that represents how items can be exchanged among agents without losing Pareto optimality. 
For an allocation $A$ and a price vector $\bfp$, we define the MBB network $D(A,\bfp)$ to be a directed graph where the vertices are given by the agents $N$ and items $O$, and the arcs are given as follows: 
\begin{itemize}
\item there is an arc from agent $i$ to item $o$ if $o \in A_i$; and  
\item there is an arc from item $o$ to agent $i$ if $o \in \MBB_i(\bfp) \setminus A_i$.
\end{itemize}
We say that $j \in N \cup O$ can {\em reach} $i \in N \cup O$ if there is a directed path from $j$ to $i$ in $D(A,\bfp)$. For $X \subseteq N$, we write $\MBB(\bfp,X)=\bigcup_{i \in X}\MBB_i(\bfp)$.

For a directed path $P$ starting from agent $j$ in a MBB network, we denote by $o(j,P)$ the item owned by $j$ on the path $P$, that is, $(j,o(j,P)) \in P$.  
For each agent $i \in N$, we say that $j \in N$ is an {\em $\epsilon$-path-violator} for $i$ if there is a directed path $P$ from $j$ to $i$ in $D(A,\bfp)$ and $j$'s weighted spending is greater than the $\epsilon$-approximate weighted spending of $i$ even after removing $o(j,P)$ from $j$'s bundle, i.e., 
$$
\frac{1}{w_j} p(A_j \setminus \{o(j,P)\}) > \frac{1+\epsilon}{w_i}p(A_i).
$$

\citet{barman2018finding} showed that the number of ``swap operations'', which reassign items along paths, is bounded by a polynomial in the input size based on a potential function argument. 
Like Barman et al., we define a potential function $\Phi_{i}$ with respect to each agent $i$. 
Let $i \in N$ be a ``root agent''. For each $j \in N \setminus \{i\}$, we define the {\em level} of $j$, denoted by $h_i(j,A,\bfp)$, to be half of the length of a shortest path from $j$ to $i$ if there is a directed path from $j$ to $i$ in $D(A,\bfp)$, and $n-1$ otherwise. 
We say that $o \in A_j$ is $j$'s {\em critical item} with respect to $i$ if there is a shortest path $P$ in $D(A,\bfp)$ from $j$ to $i$ where $o=o(j,P)$; we let $C_i(j,A,\bfp)$ be the set of critical items of $j$ with respect to $i$. We now define our potential function $\Phi_i$. For each pair of allocation $A$ and price vector $\bfp$, we let
\begin{equation}\label{eq:potential}
\Phi_{i}(A,\bfp):=\sum_{j \in N \setminus \{i\}} g_i(j,A,\bfp),
\end{equation}
where $g_i(j,A,\bfp):=m(n- h_i(j,A,\bfp)) + |C_{i}(j,A,\bfp)|$ for each $j \in N \setminus \{i\}$. It is easy to see that $\Phi_i$ is always non-negative and bounded above by $mn^2$ because $1 \leq h_i(j,A,\bfp) \leq n-1$ and $0 \leq |C_{i}(j,A,\bfp)| \leq m$. Intuitively, $g_i(j,A,\bfp)$ lexicographically orders allocations under prices $\bfp$: it decreases when $(1)$ agent $j$ gets further away from $i$ in the MBB network, or $(2)$ the distance between $i$ and $j$ does not change but the number of critical items possessed by $j$ gets smaller. This leads to the following lemma: each reassignment of a critical item to an agent who is closer to the root agent results in a decrease in the potential function \citep{barman2018finding}. 

\begin{lemma}[Proof of Lemma $13$ in the extended version of \citep{barman2018finding}]\label{lem:potential}
Given an allocation $A$ and a price vector $\bfp$, suppose that $P=(j,o_1,i_1,\ldots,o_k,i)$ is a shortest path from $j$ to $i$ in $D(A,\bfp)$. Let $A'$ be the allocation resulting from reassigning $o_1$ from $j$ to $i_1$, i.e., $A'_j=A_j \setminus \{o_1\}$, $A'_{i_1}=A_{i_1} \cup \{o_1\}$, and $A'_{i'}=A_{i'}$ for all $i' \in N \setminus \{j,i_1\}$. Then $\Phi_{i}(A,\bfp)-1 \geq \Phi_{i}(A',\bfp)$. 
\end{lemma}

\subsection{Algorithm}
We are now ready to present the algorithm (Algorithm \ref{alg:POwEF1}). To bound the number of steps, we assume for the time being that the input valuations as well as weights are integral powers of some positive value $(1+\epsilon)$ with $\epsilon\in(0,1)$; later, we will show that for the $\epsilon$-rounded version of a given instance, the algorithm returns an allocation that is Pareto optimal and $\wefi$ for the original instance if $\epsilon$ is small enough. 

\begin{algorithm}                      
\caption{Algorithm for constructing a $\po$ and $\wefi$ allocation}         
\label{alg:POwEF1}                          
\begin{algorithmic}[1]
\REQUIRE For each $o \in O$, $\exists i \in N$ with $v_i(o)>0$ (so that initial prices are positive). 
The valuations are integral powers of $(1+\epsilon)$ for some $\epsilon\in(0,1)$ (i.e., for each $o \in O$ and $i \in N$, there exist integers $a_{io}\geq 0$ and $b_i\geq 0$ such that $v_i(o)=(1+\epsilon)^{a_{io}}$ and $w_i=(1+\epsilon)^{b_i}$).  
\STATE $p_o \leftarrow \max_{i \in N} v_i(o)$ for each $o \in O$. 
\STATE Assign each item $o$ to an agent who values it most. 
\STATE Initialize $i^* \in \argmin_{i \in N} \frac{1}{w_i} p(A_i)$.
\WHILE{$A$ is not $3\epsilon$-$\wpefi$ with respect to $\bfp$} \label{line:bigwhile}
\IF{$i^*$ is not a least weighted spender}
\STATE Update $i^* \in \argmin_{i \in N} \frac{1}{w_i} p(A_i)$.\label{line:weightedleastspender}
\ENDIF\\
\textbf{\color{red}  /*Swap phase*/}
\IF{there is an $\epsilon$-path-violator for $i^*$ in $D(A,p)$}\label{line:swap}
\STATE Choose an $\epsilon$-path-violator $j$ who has a shortest path $P=(i_0,o_1,i_1,\ldots,o_k,i_k)$ where $i_0=j$ and $i_k=i^*$ in $D(A,\bfp)$, i.e., $\frac{1}{w_{i_0}}p(A_{i_0} \setminus \{o_{1}\}) > \frac{1+\epsilon}{w_{i_k}}p(A_{i_k})$ and $\frac{1}{w_{i_t}}p(A_{i_t} \setminus \{o_{t+1}\}) \leq \frac{1+\epsilon}{w_{i_k}}p(A_{i_k})$ for $t=1,2,\ldots,k-1$.
\STATE $h \leftarrow 0$
\WHILE{$\frac{1}{w_{i_h}}p(A_{i_h} \setminus \{o_{h+1}\}) > \frac{1+\epsilon}{w_{i^*}}p(A_{i^*})$ and $h \le k-1$}\label{line:swap1}
\STATE $A_{i_h} \leftarrow A_{i_{h}} \setminus \{o_{h+1}\}$ and $A_{i_{h+1}} \leftarrow A_{i_{h+1}} \cup \{o_{h+1}\}$
\textbf{/*apply the swap operation as long as agent $i_h$ is an $\epsilon$-path-violator*/} \label{line:transfer-item}
\STATE $h \leftarrow h+1$
\ENDWHILE\label{end:swap}\\
\hspace{-4mm} \textbf{\color{red}/*Price-rise phase*/}
\ELSE
\STATE Set $I$ to be the set of agents who can reach $i^*$ in $D(A,\bfp)$ (the set $I$ includes $i^*$). \label{line:price-rise}
\STATE $x \leftarrow \frac{\wmax(A,\bfp)}{\wmin(A,\bfp)}$\label{line:x}
\STATE $y \leftarrow \min \left\{\, \frac{\alpha_i(\bfp) p_o}{v_i(o)} \,\middle|\, i \in I \land o \in O \setminus \MBB(\bfp,I) \,\right\}$\label{line:y}
\STATE $z \leftarrow \min \left\{\, \frac{p(A_j)}{w_j\wmin(A,\bfp)} \,\middle|\, j \in N \setminus I\,\right\}$\label{line:z}
\IF{$\wmin(A,\bfp)>0$ and $x=\min \{x,y,z\}$}
\STATE $\gamma \leftarrow x$\label{line:price-rise:x}
\textbf{/*raise prices of items in $\MBB(\bfp,I)$ until $A$ becomes $3\epsilon$-$\wpefi$*/}\\
\ELSIF{$\wmin(A,\bfp)=0$ or $y=\min \{x,y,z\}$}\label{line:minmax}
\STATE $\gamma \leftarrow y$\label{line:price-rise:y}
\textbf{/*raise prices of items in $\MBB(\bfp,I)$ until a new MBB edge between $I$ and $O \setminus \MBB(\bfp,I)$ appears*/}\\
\ELSE 
\STATE \textbf{/*$\wmin(A,\bfp)>0$ and $z=\min \{x,y,z\}$*/}
\STATE $\gamma \leftarrow (1+\epsilon)$ \label{line:price:z}
\textbf{/*raise prices of items in $\MBB(\bfp,I)$ by a factor of $(1+\epsilon)$ if some agent outside $I$ becomes the least weighted spender when raising prices by $x$ or $y$*/}
\ENDIF
\STATE $p_o \leftarrow \gamma p_o$ for each $o \in \MBB(\bfp,I)$\label{line:gamma}
\ENDIF
\ENDWHILE
\RETURN $A$
\end{algorithmic}
\end{algorithm}

\subsubsection{Overview of Algorithm \ref{alg:POwEF1}}
Initially, the algorithm allocates each item to an agent who values it most, and sets the price of each item to the value that the assigned agent derives from the item. At this point, the bang per buck ratio of each item for the assigned agent is $1$, and agents receive their MBB items only. The initial prices are positive if we assume that for each item, there is an agent who has a positive value for it (any item that has zero value to all agents can be thrown away without loss of generality). The algorithm then iterates between a {\em swap phase} (from Line \ref{line:swap} to Line \ref{end:swap}) and a {\em price-rise phase} (from Line \ref{line:price-rise} to Line \ref{line:gamma}).

In the  {\em swap phase}, the algorithm transfers items along a path to the least weighted spender from a corresponding $\epsilon$-path-violator, reducing inequality in the weighted spending under the fixed prices. The transfer continues as long as the agent who gets a new item is an $\epsilon$-path-violator (Line \ref{line:swap1}). If no item can be transferred from a violator to the least weighted spender in the current MBB network (namely, the least weighted spender has no $\epsilon$-path-violator), the algorithm enters the price-rise phase. 

In the {\em price-rise phase}, we aim to increase reachability between the least weighted spender and items possessed by violators. To this end, the algorithm uniformly raises prices of MBB items $\MBB(\bfp,I)=\bigcup_{i \in I}\MBB_i(\bfp)$ for the set $I$ of agents who can reach the least weighted spender $i^*$. Specifically, as long as $i^*$ remains the least weighted spender, the algorithm raises prices until one of the following holds:  
\begin{itemize}
\item $\wmin(A,\bfp)$ exceeds $\wmax(A,\bfp)$ (in the case where the least weighted spending is positive and $x=\min\{x,y,z\}$); or
\item a new MBB edge appears between the agents in $I$ and the items in $O \setminus \MBB(\bfp,I)$ (in the case where the least weighted spending is $0$ or $y=\min\{x,y,z\}$),
\end{itemize}
where $x$, $y$, and $z$ are the multiplicative factors defined in Line \ref{line:x}, \ref{line:y}, and \ref{line:z}, respectively.
Note that as all prices are positive, and $O \setminus \MBB(\bfp,I)$ and $N \setminus I$ are nonempty (we will prove this in Lemma~\ref{lem:MBB:price}), $y$ and $z$ are well-defined.
If the identity of $i^*$ changes when increasing prices by a factor of $x$ or $y$ (i.e., the least weighted spending is positive and $z=\min\{x,y,z\}$), then the algorithm raises prices by $(1+\epsilon)$. The scaling with $x$ ensures that the resulting outcome is $3\epsilon$-$\wpefi$ (see the subsequent proof for Corollary \ref{cor:x:termination}). Thus, up until the final step, the algorithm continues to raise prices by either $y$ or $(1+\epsilon)$, both of which are powers of $(1+ \epsilon)$ during the execution of the algorithm. By doing so, the algorithm keeps the invariant that the prices are powers of $(1+\epsilon)$. We will argue that $\wmax(A,\bfp)$ never increases while $\wmin(A,\bfp)$ increases by a factor of $(1+\epsilon)$ after a polynomial number of steps, thereby showing that the difference between the value of $\wmin(A,\bfp)$ and that of $\wmax(A,\bfp)$ becomes negligibly small after a finite number of steps.

\subsubsection{Swap phase}
First, consider the swap phase starting with Line \ref{line:swap} and ending with Line \ref{end:swap} of Algorithm \ref{alg:POwEF1}. Suppose that $P=(i_0,o_1,i_1,\ldots,o_k,i_k)$ is the shortest path selected by the algorithm from the $\epsilon$-path-violator $j=i_0$ to the least weighted spender $i^*=i_k$. Let $A'$ be the resulting allocation after swapping items along the path $P$ by applying Lines \ref{line:swap1}--\ref{end:swap}; see Figure \ref{fig:swap} for an illustration. An immediate observation is that the path-violator loses at least one item after the swap:

\begin{lemma}\label{lem:pathviolator}
$A'_{i_0} \subsetneq A_{i_0}$.
\end{lemma}

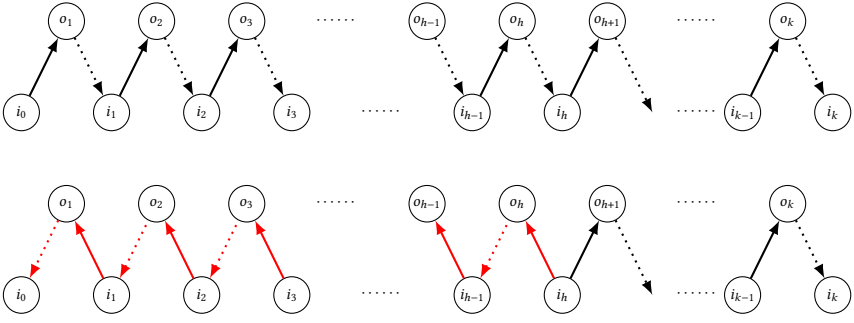
\begin{figure}[thb]
\centering
\begin{tikzpicture}[scale=0.6, transform shape, every node/.style={minimum size=8mm, inner sep=1pt}]
	\node[draw, circle](1) at (-6,0) {$i_0$};
	\node[draw, circle](2) at (-4,0) {$i_1$};
	\node[draw, circle](3) at (-2,0) {$i_2$};
	\node[draw, circle](4) at (0,0) {$i_3$};
	\node at (2,0) {$\cdots \cdots$};
	\node[draw, circle](5) at (4,0) {$i_{h-1}$};
	\node[draw, circle](6) at (6,0) {$i_h$};
	\node at (9,0) {$\cdots \cdots$};
	\node[draw, circle](ik1) at (10,0) {$i_{k-1}$};
	\node[draw, circle](ik) at (12,0) {$i_k$};
	
	\node[draw, circle](o1) at (-5,2) {$o_1$};
	\node[draw, circle](o2) at (-3,2) {$o_2$};
	\node[draw, circle](o3) at (-1,2) {$o_3$};
	\node at (1,2) {$\cdots \cdots$};
	\node[draw, circle](o4) at (3,2) {$o_{h-1}$};
	\node[draw, circle](o5) at (5,2) {$o_h$};
	\node[draw, circle](o6) at (7,2) {$o_{h+1}$};
	\node at (9,2) {$\cdots \cdots$};
	\node[draw, circle](ok) at (11,2) {$o_{k}$};
	
	\draw[->, >=latex,thick] (1)--(o1);
	\draw[->, >=latex,thick,dotted] (o1)--(2);
	\draw[->, >=latex,thick] (2)--(o2);
	\draw[->, >=latex,thick,dotted] (o2)--(3);
	\draw[->, >=latex,thick] (3)--(o3);
	\draw[->, >=latex,thick,dotted] (o3)--(4);
	\draw[->, >=latex,thick,dotted] (o4)--(5);
	\draw[->, >=latex,thick] (5)--(o5);
	\draw[->, >=latex,thick,dotted] (o5)--(6);
	\draw[->, >=latex,thick] (6)--(o6);
	\draw[->, >=latex,thick,dotted] (o6)--(8,0);
	\draw[->, >=latex,thick] (ik1)--(ok);
	\draw[->, >=latex,thick,dotted] (ok)--(ik);
	
\begin{scope}[yshift=-4cm]
\node[draw, circle](1) at (-6,0) {$i_0$};
	\node[draw, circle](2) at (-4,0) {$i_1$};
	\node[draw, circle](3) at (-2,0) {$i_2$};
	\node[draw, circle](4) at (0,0) {$i_3$};
	\node at (2,0) {$\cdots \cdots$};
	\node[draw, circle](5) at (4,0) {$i_{h-1}$};
	\node[draw, circle](6) at (6,0) {$i_h$};
	\node at (9,0) {$\cdots \cdots$};
	\node[draw, circle](ik1) at (10,0) {$i_{k-1}$};
	\node[draw, circle](ik) at (12,0) {$i_k$};
	
	\node[draw, circle](o1) at (-5,2) {$o_1$};
	\node[draw, circle](o2) at (-3,2) {$o_2$};
	\node[draw, circle](o3) at (-1,2) {$o_3$};
	\node at (1,2) {$\cdots \cdots$};
	\node[draw, circle](o4) at (3,2) {$o_{h-1}$};
	\node[draw, circle](o5) at (5,2) {$o_h$};
	\node[draw, circle](o6) at (7,2) {$o_{h+1}$};
	\node at (9,2) {$\cdots \cdots$};
	\node[draw, circle](ok) at (11,2) {$o_{k}$};
	
	\draw[->, >=latex,thick,dotted,red] (o1)--(1);
	\draw[<-, >=latex,thick,red] (o1)--(2);
	\draw[<-, >=latex,thick,dotted,red] (2)--(o2);
	\draw[<-, >=latex,thick,red] (o2)--(3);
	\draw[<-, >=latex,thick,dotted,red] (3)--(o3);
	\draw[<-, >=latex,thick,red] (o3)--(4);
	\draw[<-, >=latex,thick,red] (o4)--(5);
	\draw[<-, >=latex,thick,dotted,red] (5)--(o5);
	\draw[<-, >=latex,thick,red] (o5)--(6);
	\draw[->, >=latex,thick] (6)--(o6);
	\draw[->, >=latex,thick,dotted] (o6)--(8,0);
	\draw[->, >=latex,thick] (ik1)--(ok);
	\draw[->, >=latex,thick,dotted] (ok)--(ik);
\end{scope}	
\end{tikzpicture}
\caption{Example of a swap operation. The thick lines correspond to the pairs $(i_t,o_{t+1})$ such that agent $i_t$ owns item $o_{t+1}$; the dotted lines correspond to the pairs $(o_{t},i_t)$ such that item $o_{t}$ is an MBB item for $i_t$ but is not owned by $i_t$. The upper figure (respectively, the lower figure) represents the MBB network before the swap operation (respectively, after the swap operation) along $P=(i_0,o_1,i_1,\ldots,o_k,i_k)$.}
\label{fig:swap}
\end{figure}

Further, the swap operation preserves the property that each intermediate agent on the path $P$ is a non-violator but their weighted spending does not go below the minimum weighted spending $\wmin(A,\bfp)$ with respect to $A$. 
\begin{lemma}\label{lem:swap}
For each $h=1,\ldots,k$, we have $\frac{1}{w_{i_h}} p(A'_{i_h}) \geq \wmin(A,\bfp)$ and $(1+\epsilon)\wmin(A,\bfp) \geq \min_{o \in A'_{i_h}}\frac{1}{w_{i_h}} p(A'_{i_h} \setminus \{o\})$. 
\end{lemma}
\begin{proof}
The claim trivially holds for the least weighted spender $i_k=i^*$ as she receives at most one item $o_k$ and does not lose any item by the swap operation. Consider agent $i_h$ with $0< h < k$. By the choice of $P$, agent $i_h$ is not an $\epsilon$-path-violator in $A$. Thus, removing $o_{h+1}$ from the bundle of $i_h$ decreases its weighted value to or below $(1+\epsilon)\wmin(A,\bfp)$, i.e., 
\begin{equation}\label{eq:1}
(1+\epsilon)\wmin(A,\bfp) \geq  \frac{1}{w_{i_h}}p(A_{i_h} \setminus \{o_{h+1}\}).
\end{equation}
Now consider the three cases
\begin{itemize}
\item[$(1)$] $A'_{i_h}= A_{i_h}$; 
\item[$(2)$] $A'_{i_h}= A_{i_h} \cup \{o_{h}\}$; 
\item[$(3)$] $A'_{i_h}= (A_{i_h}\cup \{o_{h}\}) \setminus \{o_{h+1}\}$. 
\end{itemize}

In case $(1)$, the claim clearly holds due to the fact that $\frac{1}{w_{i_h}}p(A_{i_h}) \geq \wmin(A,\bfp)$ and inequality \eqref{eq:1}. In case $(2)$, the claim holds because of the fact that $\frac{1}{w_{i_h}}p(A'_{i_h}) \geq\frac{1}{w_{i_h}}p(A_{i_h}) \geq \wmin(A,\bfp)$ and the {\bf while}-condition in Line \ref{line:swap1}. In case $(3)$, we have $A_{i_h} \setminus \{o_{h+1}\}=A'_{i_h} \setminus \{o_h\}$. Combining this with inequality \eqref{eq:1} yields
\begin{align*}
(1+\epsilon)\wmin(A,\bfp) &\geq \frac{1}{w_{i_h}} p(A_{i_h} \setminus \{o_{h+1}\})\\
&= \frac{1}{w_{i_h}} p(A'_{i_h} \setminus \{o_h\}) \geq \min_{o \in A'_{i_h}} \frac{1}{w_{i_h}} p(A'_{i_h} \setminus \{o\}). 
\end{align*}
By the {\bf while}-condition in Line \ref{line:swap1}, we further have $\frac{1}{w_{i_h}} p(A'_{i_h})=  \frac{1}{w_{i_h}} p((A_{i_h}\cup \{o_{h}\}) \setminus \{o_{h+1}\})\geq (1+\epsilon)\wmin(A,\bfp)$. This proves the claim. 
\end{proof}

As a corollary of the above lemmas, assuming that $\wmax(A,\bfp) \geq (1+\epsilon)\wmin(A,\bfp)$ (which always holds before a swap phase due to the {\bf while}-condition in Line~\ref{line:bigwhile}), the function $\wmax$ does not increase due to the swap operations along $P$; further, the weighted spending $\wmin(A,\bfp)$ of the least weighted spender does not decrease, and strictly increases only if agent $i_k$ receives item $o_k$. 
\begin{corollary}\label{cor:swap:wmax}
Suppose that $\wmax(A,\bfp) \geq (1+\epsilon)\wmin(A,\bfp)$. Then, $\wmax(A,\bfp) \geq \wmax(A',\bfp)$. 
\end{corollary}
\begin{proof}
The spending of agents who do not appear in $P$ does not change; so consider agent $i_h$ on $P$. For the $\epsilon$-path-violator $i_0$, Lemma \ref{lem:pathviolator} implies
\[
\wmax(A,\bfp) \geq \min_{o \in A_{i_0}}\frac{1}{w_{i_0}}p(A_{i_0}\setminus \{o\}) \geq \min_{o \in A'_{i_0}}\frac{1}{w_{i_0}}p(A'_{i_0}\setminus \{o\}). 
\]
For agent $i_h$ with $h>0$, by Lemma \ref{lem:swap}, we have
\[
\wmax(A,\bfp) \geq (1+\epsilon)\wmin(A,\bfp) \geq \min_{o \in A'_{i_h}}\frac{1}{w_{i_h}} p(A'_{i_h}\setminus \{o\}).
\]
This proves the claim.
\end{proof}

\begin{corollary}\label{cor:swap:wmin}
It holds that $\wmin(A,\bfp) \leq \wmin(A',\bfp)$. Further, $\wmin(A,\bfp) < \wmin(A',\bfp)$ only if the least weighted spender $i_k$ receives item $o_k$, i.e., $A'_{i_k}=A_{i_k}\cup \{o_k\}$. 
\end{corollary}
\begin{proof}
By definition, $\frac{1}{w_{i_0}}p(A'_{i_0}) \geq \wmin(A,\bfp)$. Also, by Lemma \ref{lem:swap}, $\frac{1}{w_{i_h}}p(A'_{i_h}) \geq \wmin(A,\bfp)$ for $h=1,2,\ldots,k$. Since no other agent changes her bundle, we have $\wmin(A,\bfp) \leq \wmin(A',\bfp)$. It is also trivial to see that $\wmin(A,\bfp) < \wmin(A',\bfp)$ only if agent $i_k$ receives item $o_k$, i.e., $A'_{i_k}=A_{i_k}\cup \{o_k\}$. 
\end{proof}

We now show that the potential function decreases by at least $1$ after the swap operations along~$P$. 

\begin{lemma}[Lemma $13$ in the extended version of \citep{barman2018finding}]\label{lem:potential2}
We have $\Phi_{i^*}(A,\bfp) -1 \geq \Phi_{i^*}(A',\bfp)$.
\end{lemma}
\begin{proof}
Suppose that the exchange terminates when giving item $o_{\ell}$ to agent $i_{\ell}$ with $1 \leq \ell \leq k$. Applying Lemma \ref{lem:potential} repeatedly to the pair of agents $i_h$ and $i_{h+1}$ for $0 \leq h \leq \ell -1$, we get that $\Phi_{i^*}(A,\bfp) -1 \geq \Phi_{i^*}(A',\bfp)$. 
\end{proof}

\subsubsection{Price-rise phase}
Next, let us consider the price-rise phase starting with Line \ref{line:price-rise} and ending with Line \ref{line:gamma} of Algorithm \ref{alg:POwEF1}. 
Let $(A,\bfp)$ be the pair of allocation and price vector just before the price-rise phase. Let $i^*$ be the corresponding least weighted spender defined in Line~\ref{line:weightedleastspender}, $I$ be the set of agents who can reach $i^*$ ($I$ includes $i^*$), and $x,y,z$ be the multiplicative factors defined in Lines~\ref{line:x}, \ref{line:y}, and \ref{line:z}. Furthermore, let 
\[
\gamma :=
\begin{cases}
x~&\mbox{if $\wmin(A,\bfp)>0$ and $x=\min \{x,y,z\}$};\\
y~&\mbox{if $\wmin(A,\bfp)=0$ or $y=\min \{x,y,z\}$};\\
1+\epsilon~&\mbox{otherwise}.
\end{cases}
\]
Note that if $\wmin(A,\bfp)=0$, then we always set $\gamma=y$, so we do not have to worry about the well-definedness of $x$ in this case.

We assume that up until this point, 
\begin{itemize}
\item for each item $o \in O$, $p_o=(1+\epsilon)^a$ for some non-negative integer $a$; and 
\item $A$ is not $3\epsilon$-$\wpefi$ with respect to $\bfp$, i.e., $\wmax(A,\bfp) > (1+ 3\epsilon)\wmin(A,\bfp)$; and
\item $A$ is complete, and $A_i \subseteq \MBB_i(\bfp)$ for each $i \in N$.
\end{itemize}
Let $\bfp'$ be the new price vector after the price-update where $p'_o=\gamma p_o$ for each $o \in \MBB(\bfp,I)$ and $p'_o=p_o$ for each $o \in O \setminus \MBB(\bfp,I)$. 
We prove the following auxiliary lemmas. First, we observe that since $(A,\bfp)$ admits a violator, there is an item that will not experience the price-update.    

\begin{lemma}\label{lem:MBB:price}
$O \setminus \MBB(\bfp,I)\neq \emptyset$. In particular, $N \setminus I \neq \emptyset$, and $O \setminus \MBB_i(\bfp) \neq \emptyset$ for each $i \in I$. 
\end{lemma}
\begin{proof}
Suppose towards a contradiction that $O \setminus \MBB(\bfp,I) = \emptyset$, i.e., $\MBB(\bfp,I) = O$. Take an $\epsilon$-violator $j$ (so $A_j \neq \emptyset$) such that $\wmax(A,\bfp)= \min_{o \in A_j}p(A_j \setminus \{o\})$. Then, since every item in $A_j$ belongs to $\MBB(\bfp,I)$, $j$ has a directed path to the least weighted spender $i^*$ via any item in $A_j$. Thus, $j \in I$. This means that $j$ is an $\epsilon$-path-violator for $i^*$ as the weighted spending of $j$ is strictly higher than $(1+ \epsilon)\wmin(A,\bfp)$ even after removing any item from $A_j$. However, by the {\bf if}-condition in Line \ref{line:swap}, $j$ must not be an $\epsilon$-path-violator for $i^*$, a contradiction. 
Since every agent only receives her MBB items in $A$, and $A$ is a complete allocation, we have $N \setminus I \neq \emptyset$.
\end{proof}

Another observation is that the price of each item in $\MBB(\bfp,I)$ increases by a power of $(1+ \epsilon)$ when $\gamma \neq x$, and the price of each item not in $\MBB(\bfp,I)$ does not change after the price-rise phase. Further, the price-rise phase makes MBB edges from items in $\MBB(\bfp,I)$ to agents outside $I$ vanish while preserving the other edge structures. 

\begin{lemma}\label{lem:xy}
The following properties hold:
\begin{itemize}
\item[$($i$)$] $y=(1+ \epsilon)^a$ for some positive integer $a$.  
\item[$($ii$)$] For each $i \not \in I$ and each $o \in A_i$, it holds that $p_o =p'_o$. 
\item[$($iii$)$] For each $o \in \MBB(\bfp,I)$ and each $i  \in N \setminus I$ with $A_i \neq \emptyset$, item $o$ does not have a directed path to agent $i$ in $D(A,\bfp')$. In particular, $o \not \in \MBB_i(\bfp')$. 
\end{itemize}
\end{lemma}
\begin{proof}
\begin{itemize}
\item[$($i$)$] Observe that by definition of the maximum bang per buck ratio, $\alpha_i(\bfp) > \frac{v_i(o)}{p_o}$ for any $i \in N$ and $o \not \in \MBB_i(\bfp)$; from Lemma \ref{lem:MBB:price} we have $y>1$. Since the prices and valuations are assumed to be integral powers of $(1+ \epsilon)$, $y$ is also an integral power of $(1+ \epsilon)$ and thus the statement $($i$)$ holds. 
\item[$($ii$)$] We know that $p_o \neq p’_o$ only if $o \in \MBB(\bfp,I)$. If some $o \in A_i$ with $i \not \in I$ belongs to $\MBB(\bfp,I)$, namely $o \in \MBB_{i'}(\bfp)$ for some agent $i' \in I$, then $i$ has a directed path to the least weighted spender $i^*$ which starts with the edges $(i,o)$ and $(o,i')$ followed by a directed path from $i'$ to $i^*$, a contradiction. 
\item[$($iii$)$] We will first show that $o \not \in \MBB_i(\bfp')$ for each $o \in \MBB(\bfp,I)$ and each $i \not \in I$ with $A_i \neq \emptyset$. 
Suppose otherwise, i.e., $o \in \MBB_i(\bfp')$ for some $o \in \MBB(\bfp,I)$ and some $i \not \in I$ with $A_i \neq \emptyset$. 
Let $o' \in A_i$. Then, since $p'_{o'}=p_{o'}$ by $($ii$)$ and $o' \in \MBB_i(\bfp)$ by our assumption, we have $\frac{v_{i}(o')}{p'_{o'}}=\frac{v_{i}(o')}{p_{o'}}=\alpha_i(\bfp)$. However, recall that $\gamma>1$ since $\min \{x,y,1+\epsilon\}>1$. This implies $p'_{o}=\gamma p_o > p_o$, and
\[
\frac{v_{i}(o)}{p'_o} < \frac{v_{i}(o)}{p_o} \leq \alpha_i(\bfp) = \frac{v_{i}(o')}{p'_{o'}} \leq \alpha_i(\bfp'),
\]
a contradiction. 

Now, we will show that for each $o \in \MBB(\bfp,I)$ and each $i \not \in I$ with $A_i \neq \emptyset$, there is no directed path from $o$ to $i$ in $D(A,\bfp')$. Assume towards a contradiction that there is a directed path $P$ from item $o  \in \MBB(\bfp,I)$ to agent $i \not \in I$ in $D(A,\bfp')$; let $P=(o_1,i_1,o_2, \ldots,o_k,i_k)$ where $o_1=o$ and $i_k=i$. We will show by induction that $i_t \in I$ for each $t=1,2,\ldots,k$. Observe first that by the above reasoning, item $o$ will not be desired by the agents outside $I$ after the price-rise phase, namely, $o \not \in \MBB_{i'}(\bfp')$ for any $i' \in N \setminus I$. Thus, $i_1 \in I$. Suppose that $i_{t'} \in I$ for each $1 \leq t' \leq t$; we will prove that $i_{t+1} \in I$. Since agent $i_t$ has an edge towards $o_{t+1}$ in $D(A,\bfp')$, we have $o_{t+1} \in A_{i_t} \subseteq \MBB_{i_t}(\bfp)$. Thus, $o_{t+1} \in \MBB(\bfp,I)$; by applying the same reasoning as the above observation about item $o$, this implies $i_{t+1} \in I$, completing the induction. But this means $i=i_k  \in I$, contradicting the assumption that $i \not \in I$. \qedhere
\end{itemize}
\end{proof}

We also observe that the identity of the least weighted spender does not change when $\gamma \in \{x,y\}$. 

\begin{lemma}\label{lem:identity:leastspender}
If $\min \{x,y\} \leq z$, then $i^*$ remains a least weighted spender with respect to $\bfp'$, i.e., $\frac{1}{w_{i^*}}p'(A_{i^*})=\wmin(A,\bfp')$. 
\end{lemma}
\begin{proof}
If $p(A_{i^*})=0$, then the weighted spending of $i^*$ remains $0$ after the price-update, and thus the claim holds. So suppose $p(A_{i^*})>0$. Since the prices of the items owned by agents in $I$ uniformly increased, $\frac{p'(A_{i^*})}{w_{i^*}} \leq \frac{p'(A_{i})}{w_{i}}$ for any $i \in I$. Also, by definition of $z$, for any $j \in N \setminus I$ we have
\[
\frac{p'(A_{i^*})}{w_{i^*}}= \gamma\cdot \frac{p(A_{i^*})}{w_{i^*}} \leq z \cdot \wmin(A,\bfp) \leq \frac{p(A_j)}{w_j\wmin(A,\bfp)}\cdot \wmin(A,\bfp)=\frac{p'(A_j)}{w_j}, 
\]
which implies that $i^*$ remains a least weighted spender with respect to $\bfp'$. 
\end{proof}

We now prove that $A$ still satisfies the MBB condition with respect to the new price vector. 

\begin{lemma}\label{lem:price:MBB}
For each agent $i \in N$, it holds that $A_i \subseteq \MBB_i(\bfp')$. 
\end{lemma}
\begin{proof}
Take any $i \in N$; by our assumption, $A_i \subseteq \MBB_i(\bfp)$. Consider the following two cases. First, suppose $i \not \in I$. Then, $p_o=p'_o$ for each $o \in A_i$ by Lemma \ref{lem:xy}$($ii$)$. Thus, the price of every $o \in A_i$ remains the same while the price of every other item does not decrease, so $A_i \subseteq \MBB_i(\bfp')$. Second, suppose $i \in I$. Then, for any $o \in \MBB_i(\bfp)$ and $o' \in O \setminus \MBB(\bfp,I)$, the bang per buck ratio for $o$ at $\bfp'$ still remains higher or equal to that for $o'$, i.e., 
\[
\frac{v_i(o)}{p'_o} \geq \frac{v_i(o)}{p_o} \cdot \frac{1}{y}\geq \frac{v_i(o)}{p_o} \cdot \frac{v_i(o')}{\alpha_i(\bfp) p_{o'}}=  \frac{v_i(o')}{p'_{o'}}, 
\]
where the first inequality follows from the fact that $\gamma \leq y$ (recall that $y= (1+\epsilon)^a\geq 1+\epsilon$ from Lemma \ref{lem:xy}$($i$)$) and the second inequality follows from the definition of $y$ (namely, $\frac{1}{y} \geq \frac{v_i(o')}{\alpha_i(\bfp)p_{o'}}$). Further, for any $o \in \MBB_i(\bfp)$ and $o' \in \MBB(\bfp,I) \setminus \MBB_i(\bfp)$, the bang per buck ratio for $o$ at $\bfp'$ remains higher than that for $o'$, i.e., 
\[
\frac{v_i(o)}{p'_o} = \frac{v_i(o)}{\gamma p_o} > \frac{v_i(o')}{\gamma p_{o'}}=  \frac{v_i(o')}{p'_{o'}}. 
\]
Thus, every item in $\MBB_i(\bfp)$ still has a higher or equal bang per buck ratio as the other items, and all such items have the same bang per buck ratio since the prices of the items in $\MBB_i(\bfp)$ have been increased by the same factor $\gamma$. This means that $\MBB_i(\bfp)\subseteq \MBB_i(\bfp')$. Since $A_i \subseteq \MBB_i(\bfp)$, we conclude that $A_i  \subseteq \MBB_i(\bfp')$. 
\end{proof}

Finally, we show that the algorithm terminates if it raises prices by scaling with $x$. 

\begin{lemma}\label{lem:priceupdate:x}
If $\gamma = x$ and $\wmax(A,\bfp) \geq \wmax(A,\bfp')$, then $(A,\bfp')$ is $\wpefi$. 
\end{lemma}
\begin{proof}
Suppose that $\gamma = x$ and $\wmax(A,\bfp) \geq \wmax(A,\bfp')$. Then, the algorithm terminates by Lemma \ref{lem:identity:leastspender} and by definition of $x$. Indeed, at price vector $\bfp'$, we have
$$
\wmin(A,\bfp')=x \cdot \wmin(A,\bfp)=\wmax(A,\bfp) \geq \wmax(A,\bfp'),
$$
and thus $(A,\bfp')$ is $\wpefi$. 
\end{proof}

\begin{lemma}\label{lem:wmax:termination}
If $\wmax(A,\bfp) < \wmax(A,\bfp')$, then $(A,\bfp')$ is $3\epsilon$-$\wpefi$.
\end{lemma}
\begin{proof}
Suppose that $\wmax(A,\bfp)< \wmax(A,\bfp')$, and consider an agent $j \in N$ such that $\wmax(A,\bfp')=\frac{1}{w_j}\min_{o \in A_j} p'(A_j \setminus \{o\})$. 
If $j \not \in I$, then the prices of items owned by $j$ do not change by Lemma \ref{lem:xy}$($ii$)$ and hence $\wmax(A,\bfp')=\wmax(A,\bfp)$, a contradiction. Thus, $j \in I$, meaning that $j$ has a directed path $P$ to the least weighted spender $i^*$ in $D(A,\bfp)$. By the {\bf if}-condition in Line \ref{line:swap}, $j$ is not an $\epsilon$-path-violator with respect to $(A,\bfp)$; thus, by deleting the item $o(j,P)$ from $A_j$, $j$'s weighted spending becomes smaller than or equal to $(1+\epsilon)\wmin(A,\bfp)$, which implies 
\begin{equation}\label{eq:wmin}
(1+\epsilon)\wmin(A,\bfp) \geq \min_{o \in A_j} \frac{1}{w_j}  p(A_j\setminus \{o\}).  
\end{equation}
Consider first the case where $\wmin(A,\bfp') = \gamma \wmin(A,\bfp)$, meaning that the identity of $i^*$ did not change in going from $\bfp$ to $\bfp'$. Then, $(A,\bfp')$ is $\epsilon$-$\wpefi$ since
\begin{align*}
(1+\epsilon)\wmin(A,\bfp') &=\gamma (1+\epsilon) \wmin(A,\bfp) \\
&\geq \gamma \min_{o \in A_j} \frac{1}{w_j}  p(A_j\setminus \{o\}) \\
&=  \wmax(A,\bfp'),
\end{align*}
where the inequality follows from \eqref{eq:wmin}. 
Next, consider the case where $\wmin(A,\bfp') \neq \gamma \wmin(A,\bfp)$. By definition, $\wmin(A,\bfp') < \gamma \wmin(A,\bfp)$, which means that $\wmin(A,\bfp)>0$ and $x$ takes a finite value. In this case, the identity of $i^*$ has changed, so by Lemma \ref{lem:identity:leastspender}, $\min \{x,y\}>z$. Thus, $\gamma=(1+\epsilon)$.  
Since $\wmin(A,\bfp) \leq \wmin(A,\bfp')$, 
Inequality~\eqref{eq:wmin} implies
\[
(1+\epsilon)\wmin(A,\bfp') \geq \min_{o \in A_j} \frac{1}{w_j} p(A_j \setminus \{o\}). 
\]
Multiplying this with $\gamma=(1+\epsilon)$, we get
\begin{align*}
(1+3\epsilon)\wmin(A,\bfp') &\geq (1+\epsilon)^2\wmin(A,\bfp') \\
&\geq (1+\epsilon)\min_{o \in A_j} \frac{1}{w_j} p(A_j \setminus \{o\}) = \min_{o \in A_j} \frac{1}{w_j} p'(A_j \setminus \{o\}) =\wmax(A,\bfp'). 
\end{align*}
We conclude that $(A,\bfp')$ is $3\epsilon$-$\wpefi$, as desired.
\end{proof}

The two lemmas immediately imply that scaling with $x$ ensures the termination of the algorithm. 

\begin{corollary}\label{cor:x:termination}
If $\gamma = x$, then $(A,\bfp')$ is $3\epsilon$-$\wpefi$. 
\end{corollary}
\begin{proof}
Suppose $\gamma = x$. If $\wmax(A,\bfp) \geq \wmax(A,\bfp')$, then $(A,\bfp')$ is $\wpefi$ by Lemma~\ref{lem:priceupdate:x}. If $\wmax(A,\bfp) < \wmax(A,\bfp')$, then $(A,\bfp')$ is $3\epsilon$-$\wpefi$ by Lemma \ref{lem:wmax:termination}. This proves the claim.  
\end{proof}

\subsubsection{Convergence of Algorithm \ref{alg:POwEF1}}
Using Lemmas \ref{lem:priceupdate:x} and \ref{lem:wmax:termination}, we will show that during its course, the algorithm repeatedly increases the minimum weighted spending $\wmin(A,\bfp)$ by a factor of $(1+\epsilon)$ within polynomial time, and this spending is bounded above by the function $\wmax(A,\bfp)$, which is non-increasing.
We then show that the algorithm indeed converges to a $3\epsilon$-$\wpefi$ allocation while keeping the MBB condition. 

To begin with, we show that the MBB condition is satisfied and each price is a power of $(1+\epsilon)$ at any point in the execution of the algorithm. Further, we obtain the monotonicity of the functions $\wmax(A,\bfp)$ and $\wmin(A,\bfp)$ by the observations we made in the previous subsection. 

\begin{lemma}\label{lem:alg:invariant}
During the execution of Algorithm \ref{alg:POwEF1}, the following statements hold: 
\begin{itemize}
\item[$($i$)$] $A$ is complete and satisfies the MBB condition, i.e., $A_i \subseteq \MBB_i(\bfp)$ for any $i \in N$. 
\item[$($ii$)$] If $\bfp$ is not the price vector at termination, then the price $p_o$ of each item $o \in O$ is a power of $(1+\epsilon)$, i.e., $p_o=(1+\epsilon)^{k}$ for some non-negative integer $k$.  
\item[$($iii$)$] Except possibly in the last iteration of the algorithm, the function $\wmax(A,\bfp)$ never increases.  
\item[$($iv$)$] The function $\wmin(A,\bfp)$ never decreases.
\end{itemize}
\end{lemma}
\begin{proof}
\begin{itemize}
\item[$($i$)$] The allocation $A$ is clearly complete because each item is assigned to some agent at every step. 
Take any $i \in N$. We will inductively show that agents are kept assigned to their MBB items. 
Recall that at the initial allocation $A^{(0)}$ and price vector $\bfp^{(0)}$, every agent receives an item only if she has the highest valuation for the item, and the price for each item is set to be the maximum such value. This means that each agent $i$ has $\alpha_i(\bfp^{(0)})$ at most $1$ and is assigned to an item $o$ only if her bang per buck ratio for $o$ is exactly $1$. Thus, the claim holds for the initial allocation. Moreover, each agent remains assigned to her MBB items after the swap phase (between Line \ref{line:swap} and Line \ref{end:swap}). Also, by Lemma \ref{lem:price:MBB}, each agent remains assigned to her MBB items after the price-rise phase (between Line \ref{line:price-rise} and Line \ref{line:gamma}). This proves the claim. 

\item[$($ii$)$] Recall that the initial price of each item is set to be the maximum value of the agents for that item, which is $(1+\epsilon)^{k}$ for some $k \geq 0$. 
The price of each item never decreases, and increases by a power of $(1+\epsilon)$ at each price-rise phase by Lemma \ref{lem:xy}$($i$)$ and Corollary \ref{cor:x:termination} except possibly for the final iteration. Thus, the claim holds.  

\item[$($iii$)$] The claim immediately follows from Corollary \ref{cor:swap:wmax} and Lemma \ref{lem:wmax:termination}.

\item[$($iv$)$] The function $\wmin(A,\bfp)$ does not decrease due to Corollary \ref{cor:swap:wmin} and the fact that the spending of each agent does not decrease during the price-rise phase. \qedhere
\end{itemize}
\end{proof}

Now, we prove that not only is the function $\wmin(A,\bfp)$ weakly increasing, but it is also strictly increasing during the execution of the algorithm. We first consider the corner case where the weighted spending of the least weighted spender $i^*$ is zero. In this case, the value $\wmin(A,\bfp)$ may not increase even after the price-rise phase. However, we will show that the number of agents who can reach $i^*$ strictly increases, and eventually the algorithm transfers an item of positive price from a violator to $i^*$. 
By a \emph{time step} we refer to a point in the execution of the algorithm.

\begin{lemma}\label{lem:wmin:zero}
Suppose that at time step $t$, agent $i$ is chosen as the least weighted spender $i^*$ in Line~\ref{line:weightedleastspender} of Algorithm \ref{alg:POwEF1} but $p(A_{i})=0$. Then, after $n+1$ iterations of the \textbf{while} loop in Line~\ref{line:bigwhile}, either the (unweighted) spending of $i$ strictly increases by at least $1$, or the algorithm terminates. 
\end{lemma}

\begin{proof}
Recall that the initial prices of all items are positive, so the prices of all items are positive at time step $t$. Thus, agent $i$ is not assigned to any item at time step $t$. Further, since allocation $A$ is complete at any point of the algorithm, there is always an agent $j \neq i$ with $|A_j| \geq 2$ as long as $p(A_i)=0$ and $\wmin(A,\bfp) < \wmax(A,\bfp)$. 

We first show that whenever the algorithm performs the swap phase, $i$ receives a new item, which means that $i$'s spending strictly increases and proves the desired claim. To see this, suppose that the algorithm enters the swap phase at some time step when $p(A_i)=0$. Observe that agents with at least one item have a positive spending as all prices are positive. Hence, the swap operation along the directed path $P=(i_0,o_1,i_1,\ldots,o_k,i_k)$ ends when it gives the last item $o_k$ to $i_k=i$, since if some intermediate agent $i_h$ with $0<h<k$ owns at least two items, then she becomes a path-violator for $i$. (Recall that by definition of $D(A,\bfp)$, before the swaps are performed, each of the agents $i_0,i_1,\dots,i_{k-1}$ has at least one item.) Thus, $i$'s spending becomes positive if the algorithm enters the swap phase, and the increase is at least $1$ since the prices are integral powers of $(1+\epsilon)$.

It remains to show that after at most $n$ consecutive iterations of the price-rise phase during which $p(A_i)=0$, some violator becomes a path-violator for $i$ and hence the algorithm enters the swap phase. 
To this end, it is sufficient to show that whenever the algorithm performs the price-rise phase, the number of agents who can reach $i$ (namely, agents in $I$) strictly increases.
Suppose that the algorithm increases the price of MBB items of the set $I$ of agents who can reach $i$. 
Since the set $I$ does not include an $\epsilon$-path-violator for $i$, no agent in $I$ has more than one item: Indeed, if some agent $j\in I$ has at least two items, $j$ still has a positive weighted spending even after removing any item in $A_j$; however, at least one of $j$'s items has a directed path to $i$ since $j\in I$, meaning that $j$ is an $\epsilon$-path-violator.
Further, when $p(A_i)=0$, the algorithm raises prices of MBB items of $I$ by scaling with $y$ defined in Line \ref{line:y}.
By definition of $y$, after the price-rise phase, some item $o$ not in $\MBB(\bfp,I)$ becomes a new MBB item for an agent in $I$; hence, some agent not in $I$ can now reach $i$, while all agents in $I$ can still reach $i$. It follows that the number of agents who can reach $i$ strictly increases after each price-rise phase. We conclude that the algorithm can perform at most $n$ consecutive price-rise phases. 
\end{proof}

The following lemmas further ensure that when $\wmin(A,\bfp)>0$, the minimum weighted spending strictly increases within a polynomial number of time steps.

\begin{lemma}\label{lem:identity}
During the execution of Algorithm \ref{alg:POwEF1}, suppose that agent $i$ ceases to be the least weighted spender $i^*$ just after time step $t_1$. 
Suppose further that $i$ is chosen again later as $i^*$, and let $t_2$ be the first time step when this happens.
Let $(A,\bfp)$ (respectively, $(A',\bfp')$) be the pair of allocation and price vector at time $t_1$ (respectively, time $t_2$). 
Then, it holds that $A_i \subsetneq A'_i$ or $\frac{1}{w_i}p'(A'_i) \geq \frac{1+\epsilon}{w_i}p(A_i)$. 
\end{lemma}
\begin{proof}
Since agent $i$ is not the least weighted spender at time $t_1+1$, the weighted spending of $i$ strictly increases in going from $t_1$ to $t_1+1$, either by receiving a new item in Line \ref{line:transfer-item} or by the price-rise of her items in Line \ref{line:gamma}. 

First, consider the case when $i$ does not lose any item between $t_1$ and $t_2$. Then, the claim clearly holds when $i$ gains some new item at $t_1$. Thus, suppose that the algorithm raises prices at time $t_1$. By Corollary \ref{cor:x:termination}, it raises prices by a factor of either $y$ or $(1+\epsilon)$, both of which are powers of $(1+\epsilon)$. This implies that the weighted spending of $i$ strictly increases at least by a factor of $(1+ \epsilon)$ in going from $t_1$ to $t_2$. 

Now consider the case when $i$ loses some item between $t_1$ and $t_2$, which means that $i$ becomes an $\epsilon$-path-violator for the least weighted spender $i^*$ at some time step between $t_1$ and $t_2$. Let $t'$ be the last time step before $t_2$ when $i$ becomes an $\epsilon$-path-violator and loses an item $o$. Then, $i$'s weighted spending after giving item $o$ to her neighbor is strictly higher than the value of $(1+\epsilon)\wmin(A,\bfp)$ at time step~$t'$, which is greater than or equal to $(1+\epsilon)$ times $i$'s weighted spending at time $t_1$ by Lemma \ref{lem:alg:invariant}$($iv$)$. Further, the weighted spending of $i$ does not decrease between $t'$ and $t_2$. Hence, $i$'s weighted spending increases by a factor of at least $(1+\epsilon)$.  
\end{proof}

Due to Lemma \ref{lem:potential2}, the number of consecutive swap phases is bounded by a polynomial in the input size.

\begin{lemma}\label{lem:identity:swap}
During the execution of Algorithm \ref{alg:POwEF1}, the number of consecutive swap phases during which $i^*$ remains the chosen least weighted spender is at most $mn^2$. 
\end{lemma}
\begin{proof}
Let $i$ be the least weighted spender $i^*$ chosen in Line \ref{line:weightedleastspender} of Algorithm \ref{alg:POwEF1} at some time step~$t$. By Lemma \ref{lem:potential2}, each swap operation between Line \ref{line:swap1} and Line \ref{end:swap} causes a drop of at least $1$ in the value of the function $\Phi_{i}(A,\bfp)$, which is non-negative and bounded above by $mn^2$, implying that Algorithm \ref{alg:POwEF1} can perform at most $mn^2$ iterations of the swap phase during which $i$ is the least weighted spender. 
\end{proof}

\begin{lemma}\label{lem:wmin:increase}
During the execution of Algorithm \ref{alg:POwEF1}, after $\poly(m,n)$ time, either the algorithm terminates, or $\wmin(A,\bfp)$ strictly increases by a factor of at least $(1+\epsilon)$ . 
\end{lemma}
\begin{proof}
Note that each swap and price-rise phase can be executed in $\poly(m,n)$ time: for the swap phase, finding a nearest $\epsilon$-path-violator $j$ as well as the associated path $P$ from $j$ to $i^*$ can be done in polynomial time (see Sections $4$ and $A.1$ in the extended version of \citep{barman2018finding} for details).
By Lemma \ref{lem:wmin:zero}, the minimum weighted spending becomes strictly positive after $n+1$ iterations of the \textbf{while} loop when $\wmin(A^{(0)},\bfp^{(0)})=0$ at the initial pair of allocation and price vector $(A^{(0)},\bfp^{(0)})$, so suppose without loss of generality that $\wmin(A^{(0)},\bfp^{(0)})>0$. 

Consider an arbitrary time step $t$ during the execution of the algorithm. If the algorithm performs the price-update without changing the identity of $i^*$ chosen in Line \ref{line:weightedleastspender} of Algorithm \ref{alg:POwEF1}, then the minimum weighted spending $\wmin(A,\bfp)$ increases by a factor of at least $(1+\epsilon)$ (if the algorithm raises prices by  a factor of $x$, then by Corollary~\ref{cor:x:termination}, the algorithm terminates immediately afterwards). Thus, consider the case where the identity of $i^*$ changes at every price-update. In this case, by Lemma \ref{lem:identity:swap}, the identity of $i^*$ must change after at most $mn^2$ consecutive swap phases. Observe that, after $(m+1)n$ identity changes, some agent $i$ becomes $i^*$ at least $m+1$ times by the pigeonhole principle. By Lemma \ref{lem:identity} and the fact that the size of $i$'s bundle can grow up to at most $m$, $i$'s weighted spending must increase by a factor of at least $(1+\epsilon)$ within $\poly(m,n)$ time.
\end{proof}

Let $v_{max}:=\max_{i \in N, o \in O}v_i(o)$ and $w_{max}:=\max_{i \in N}w_i$. We are ready to show the convergence of Algorithm \ref{alg:POwEF1}.

\begin{theorem}\label{thm:terminate:finite}
Algorithm \ref{alg:POwEF1} terminates in $\poly(m,n,\frac{1}{\epsilon},v_{max},w_{max})$ time. 
\end{theorem}
\begin{proof}
By Lemma \ref{lem:wmin:increase}, after $\poly(m,n)$ time, either the algorithm terminates, or the minimum weighted spending $\wmin(A,\bfp)$ strictly increases by a factor of at least $(1+\epsilon)$ (or becomes positive if it is $0$; see the proof of Lemma~\ref{lem:wmin:increase}). 
Observe that $\wmin(A,\bfp)$ is at least $\frac{1}{w_{max}}$ at the initial outcome where it is positive, and during the execution of the algorithm, $\wmin(A,\bfp)$ is at most $\wmax(A,\bfp)$, which is bounded above by $mv_{max}$ due to Lemma \ref{lem:alg:invariant}$($iii$)$ and the choice of the initial price vector. 
Now, $\log_{(1+\epsilon)}(mv_{max}/w_{max}) = \log(mv_{max}/w_{max})/\log(1+\epsilon)$, and we have $\log(1+\epsilon) > \frac{\epsilon}{1+\epsilon}$; the latter holds since the function $f(x):= \log(1+x)-\frac{x}{1+x}$ satisfies $f(0)=0$ and $f'(x) > 0$ for all $x > 0$.
The desired running time follows.
\end{proof}

Similarly to \citet{barman2018finding}, we show that the output price vector can be bounded above by a product of weights and valuations. We start by showing the following auxiliary lemma.

\begin{lemma}\label{lem:price:upperbound:zero}
During the execution of Algorithm \ref{alg:POwEF1}, suppose that the price of item $o$ has been increased with the multiplicative factor $\gamma$ (line~\ref{line:gamma}) just after time step $t$. Let $(A,\bfp)$ be the pair of allocation and price vector at time $t$. Then, it holds that $\gamma \cdot p_o \leq mv^2_{max}w_{max}$.  
\end{lemma}
\begin{proof}
Suppose that at time step $t$, agent $i$ is chosen as the least weighted spender $i^*$ in Line~\ref{line:weightedleastspender} of Algorithm \ref{alg:POwEF1}. Thus, $\frac{1}{w_i}p(A_{i})=\wmin(A,\bfp)$. 
Let $x$ and $y$ be the multiplicative factors defined in Line \ref{line:x} and Line \ref{line:y} at time step $t$, respectively. 
By definition of $\gamma$, we have $\gamma\leq y$. Indeed, if $\gamma=(1+\epsilon)$, then $(1+\epsilon) \leq y$ because $y$ is a power of $(1+\epsilon)$ by Lemma \ref{lem:xy}$($i$)$; and if $\gamma=x$, then $x \leq y$ by definition of $\gamma$. 

Let $I$ be the set of agents who can reach $i$ at time step $t$. Since $o \in \MBB(\bfp,I)$ and $A$ is a complete allocation, we have that $o \in A_{i'}$ for some $i' \in I$. Now, since the algorithm did not terminate at time step $t$, there is a violator $j$ at time step $t$ such that even after removing any item $o' \in A_{j}$ from $j$'s bundle, the weighted spending of $j$ does not go below the least weighted spending, i.e., $\frac{1}{w_j}p(A_j \setminus \{o'\})> \wmin(A,\bfp)\geq 0$. In particular, this means that $A_j$ contains at least two items. Thus, there is an item $o^*\in A_j$ such that $p_{o^*} \leq p(A_j \setminus \{o'\}) \leq w_j \wmax(A,\bfp)$, which is bounded above by $mv_{max}w_{max}$ due to Lemma \ref{lem:alg:invariant}$($iii$)$ and the choice of the initial price vector.
We further observe that $A_j \cap \MBB(\bfp,I)=\emptyset$ since otherwise $j$ would become a path-violator for $i$, contradicting the fact that the algorithm performs the price-rise phase just after time step $t$. 
Thus, by definition of $y$ and the fact that $A_j \cap \MBB(\bfp,I)=\emptyset$, we have that
\begin{equation}\label{equation:y}
y \leq \alpha_{i'}(\bfp) \frac{p_{o^*}}{v_{i'}(o^*)} = \frac{v_{i'}(o)}{p_{o}} \frac{p_{o^*}}{v_{i'}(o^*)}. 
\end{equation}
By the fact that $\gamma\leq y$ and \eqref{equation:y}, we get
\[
\gamma \cdot p_o \leq y \cdot p_o \leq \frac{v_{i'}(o)}{v_{i'}(o^*)}p_{o^*}  \leq v_{i'}(o) p_{o^*},
\]
which is bounded above by $mv^2_{max}w_{max}$ since $v_{i'}(o)\leq v_{max}$ and $p_{o^*} \leq mv_{max}w_{max}$. 
\end{proof}

We are now ready to upper-bound the final price of each item by applying the above lemma.

\begin{lemma}\label{lem:price:upperbound}
Suppose that $(A,\bfp)$ is the output of Algorithm \ref{alg:POwEF1}. Then, $p_o \leq mv^2_{max}w_{max}$ for each $o \in O$. 
\end{lemma}
\begin{proof}
If item $o$ never experiences the price-rise phase, then its final price $p_o$ is the same as the initial price, so $p_o \leq v_{max}$ by the definition of the initial price vector. Thus, suppose otherwise, namely, the price of $o$ increased at some time step. Let $t$ be the last such step. Then by Lemma \ref{lem:price:upperbound:zero}, the updated price of item $o$ is at most $mv^2_{max}w_{max}$, which proves the claim.
\end{proof}

\subsection{Approximate Instance}
\label{app:approx-instance}

As in \citet{barman2018finding}, given an arbitrary instance, we will consider an approximate version where the valuations as well as weights are integral powers of $(1+\epsilon)>1$. Formally, we define the {\em $\epsilon$-rounded instance} $\calI'=(N,O,\{v'_i\}_{i \in N},\{w'_i\}_{i \in N})$ of a given instance $\calI=(N,O,\{v_i\}_{i \in N},\{w_i\}_{i \in N})$ as follows: For each $i \in N$ and $o \in O$, the value $v'_i(o)$ is given by
\[
v'_i(o):=
\begin{cases}
(1+\epsilon)^{\ceil{\log_{1+\epsilon} v_i(o)}}~&\mbox{if $v_i(o)>0$};\\
0~&\mbox{otherwise}. 
\end{cases}
\]
For each $i \in N$, the weight $w'_i$ is given by
\[
w'_i:=(1+\epsilon)^{\floor{\log_{1+\epsilon} w_i}}
\]
Note that for each $i \in N$ and $o \in O$, we have 
\begin{equation}\label{eq:approximate:v}
v_i(o) \leq v'_i(o) \leq (1+\epsilon) v_i(o)
\end{equation}
and $\frac{1}{(1+\epsilon)}w_i \leq w'_i \leq w_i$, 
where the latter implies 
\begin{equation}\label{eq:approximate:w}
\frac{1}{w_i} \leq \frac{1}{w'_i} \leq \frac{(1+\epsilon)}{w_i}. 
\end{equation}
We will show below that the $\wefi$ property for the approximate instance translates into the $\wefi$ property for the original instance. We write the minimum positive difference of the weighted values as follows:
$$
\epsilon_{ij} := \min \left\{\, \frac{v_i(X)}{w_j} - \frac{v_i(Y) }{w_i} \,\middle|\, X,Y \subseteq O \land \frac{v_i(X)}{w_j}-\frac{v_i(Y)}{w_i}>0 \,\right\},
$$
and
\begin{equation}\label{eq:epsilon-min}
\epsilon_{min}:=\min_{i,j \in N} \epsilon_{ij}. 
\end{equation}
For the $\epsilon$-rounded instance $\calI'=(N,O,\{v'_i\}_{i \in N},\{w'_i\}_{i \in N})$ and a price vector $\bfp$, we denote the maximum bang per buck ratio for agent $i$ by $\alpha'_i(\bfp)=\max_{o \in O} \frac{v'_{i}(o)}{p_o}$ and the MBB set by
$$
\MBB'_i(\bfp)=\left\{\, o \in O \,\middle|\, \frac{v'_{i}(o)}{p_o}=\alpha'_i(\bfp) \,\right\}. 
$$

\begin{lemma}\label{lem:scale:wEF1}
Let $\calI'=(N,O,\{v'_i\}_{i \in N},\{w'_i\}_{i \in N})$ be the $\epsilon$-rounded instance of a given instance $\calI=(N,O,\{v_i\}_{i \in N},\{w_i\}_{i \in N})$. Suppose that $0<\epsilon < \min \left\{\frac{\epsilon_{min}}{63c},1\right\}$, where $c:=\max_{i \in N} \frac{v_i(O)}{w_i}$. Let $(A,\bfp)$ be a pair of a complete allocation and a price vector such that $A_i \subseteq \MBB'_i(\bfp)$ for each $i \in N$, and $A$ is $3\epsilon$-$\wpefi$ for $\calI'$ with respect to $\bfp$. Then $A$ is also $\wefi$ for $\calI$.
\end{lemma}
\begin{proof}
Take any pair of agents $i,j \in N$. Agent $i$ does not envy $j$ when $A_j=\emptyset$; so suppose $A_j\neq \emptyset$. 
Since $A$ is $3\epsilon$-$\wpefi$ for $\calI'$, there exists an item $o \in A_j$ such that
\[
\frac{1+3\epsilon}{w'_i}p(A_i) \geq \frac{1}{w'_j} p(A_j \setminus \{o\}). 
\]
Multiplying both sides of this inequality with $\alpha'_i(\bfp)$, we get 
\[
\frac{1+3\epsilon}{w'_i}v'_i(A_i) \geq \frac{1}{w'_j} \sum_{o' \in A_j \setminus \{o\}}p_{o'} \alpha'_i(\bfp) \geq \frac{1}{w'_j} v'_i(A_j \setminus \{o\}) . 
\]
Thus, we get 
\begin{align*}
&\frac{1+3\epsilon}{w'_i}v'_i(A_i) \geq \frac{1}{w'_j} v'_i(A_j \setminus \{o\})\\
&\Rightarrow \frac{(1+3\epsilon)^3}{w_i} v_i(A_i) \geq \frac{1}{w_j} v_i(A_j \setminus \{o\})\\
&\Rightarrow  \frac{63\epsilon}{w_i}v_i(A_i) \geq \frac{1}{w_j} v_i(A_j \setminus \{o\})  - \frac{1}{w_i}v_i(A_i)\\
&\Rightarrow  \epsilon_{min} >  \frac{63\epsilon}{w_i}v_i(A_i) \geq  \frac{1}{w_j} v_i(A_j \setminus \{o\})  - \frac{1}{w_i}v_i(A_i),
\end{align*}
where the second inequality follows from \eqref{eq:approximate:v} and \eqref{eq:approximate:w}, and for the third inequality we use $(1+3\epsilon)^3\leq 1+63\epsilon$, which follows from expansion and $\epsilon < 1$. By definition of $\epsilon_{min}$, we conclude that $\frac{1}{w_j} v_i(A_j \setminus \{o\}) - \frac{1}{w_i}v_i(A_i) \leq 0$, as desired.
\end{proof}

For small enough $\epsilon$, the MBB condition for $\calI'$ translates to the MBB condition for the original instance. 

\begin{lemma}[Proof of Lemma $5$ in the extended version of \citet{barman2018finding}]\label{lem:scale:PO}
Let $\bfp$ be a price vector such that $1 \leq p_o \leq \theta$ for each $o \in O$, for some $\theta>1$. Let $\calI'=(N,O,\{v'_i\}_{i \in N},\{w'_i\}_{i \in N})$ be the $\epsilon$-rounded instance of the given instance where $0< \epsilon < \frac{1}{\theta mv_{max}}$. If $A_i \subseteq \MBB'_{i}(\bfp)$ for each $i \in N$, then the allocation $A$ is Pareto-optimal for the original instance. 
\end{lemma}

\subsection{Putting Things Together: Proof of Theorem \ref{thm:POwEF1}}

We assume without loss of generality that the valuations as well as weights are at least $1$; otherwise, one can scale up these values by multiplying with the same factor $M$ and obtain an equivalent instance. Now let $\epsilon$ be such that
$$
\epsilon= \frac{1}{2}\min \left\{\frac{1}{m^2v^3_{max}w_{max}},\frac{\epsilon_{min}}{63c},1\right\},
$$
where $c:=\max_{i \in N} \frac{v_i(O)}{w_i}$ and $\epsilon_{min}$ is defined as in (\ref{eq:epsilon-min}). Consider the $\epsilon$-rounded instance $\calI'$ of the given instance $\calI$. By Lemma \ref{lem:alg:invariant}$($i$)$ and Theorem \ref{thm:terminate:finite}, Algorithm \ref{alg:POwEF1} computes a pair $(A,\bfp)$ such that $A$ is a complete allocation that satisfies $3\epsilon$-$\wpefi$ and each agent $i \in N$ receives items in $\MBB'_i(\bfp)$ for $\calI'$. By Lemmas \ref{lem:price:upperbound}, \ref{lem:scale:wEF1} and \ref{lem:scale:PO}, $A$ is $\wefi$ and Pareto optimal for $\calI$. The bound on the running time follows from Theorem \ref{thm:terminate:finite}. 
\qed

\end{document}